\documentclass{article}

\usepackage[preprint]{neurips_2025}

\usepackage[utf8]{inputenc} 
\usepackage[T1]{fontenc}    
\usepackage{hyperref}       
\usepackage{url}            
\usepackage{booktabs}       
\usepackage{amsfonts}       
\usepackage{nicefrac}       
\usepackage{microtype}      
\usepackage{xcolor}         

\usepackage{graphicx}
\usepackage{amsmath}
\usepackage{multirow}
\usepackage{booktabs}       
\usepackage{adjustbox}      
\usepackage{caption}
\usepackage{amssymb}
\usepackage{amsthm}

\newcounter{theorem}
\newtheorem{proposition}[theorem]{Proposition}

\newtheorem{thm}[theorem]{Theorem}

\title{Position: Pause Recycling LoRAs and Prioritize Mechanisms to Uncover Limits and Effectiveness}

\author{%
Mei-Yen Chen \thanks{Equal contribution} \thanks{Corresponding author: mei-yen.chen@mercedes-benz.com} $^2$ \quad Thi Thu Uyen Hoang $^{*2}$ \quad Michael Hahn $^{*1}$ \quad M. Saquib Sarfraz $^2$ \\
$^1$ Saarland University \\
$^2$ Mercedes-Benz Tech Innovation
}

\begin{document}

\maketitle

\begin{abstract}

Merging or routing low-rank adapters (LoRAs) has emerged as a popular solution for enhancing large language models, particularly when data access is restricted by regulatory or domain-specific constraints. This position paper argues that the research community should shift its focus from developing new merging or routing algorithms to understanding the conditions under which reusing LoRAs is truly effective. Through theoretical analysis and synthetic two-hop reasoning and math word problem tasks, we examine whether reusing LoRAs enables genuine compositional generalization or merely reflects shallow pattern matching. Evaluating two data-agnostic methods—parameter averaging and dynamic adapter selection—we found that reusing LoRA often fails to logically integrate knowledge across disjoint fine-tuning datasets, especially when such knowledge is underrepresented during pretraining. Our empirical results, supported by theoretical insights into LoRA's limited expressiveness, highlight the preconditions and constraints of reusing them for unseen tasks and cast doubt on its feasibility as a truly data-free approach. We advocate for pausing the pursuit of novel methods for recycling LoRAs and emphasize the need for rigorous mechanisms to guide future academic research in adapter-based model merging and practical system designs for practitioners.
\end{abstract}

\section{Introduction}
\label{sec:intro}
Many real-world applications require large language models to integrate scattered information and infer logical answers to novel questions. For instance, an AI assistant supporting human resource specialists in determining an employee's tax rate must combine information about the employee's marital status and the spouse's residency, as this affects the application of tax law. Such information often originates from scarce data sources or resides in separate systems where regulatory constraints limit direct access. Hence, compositional generalization, the ability of a model to create new combinations of known elements, is essential for the quality of such applications.

As pretraining and fine-tuning become standard practices in Large Language Model (LLM) development, reusing shared model weights from foundation models and their fine-tuned variants has emerged as a practical strategy for generalization in data-scarce scenarios. Unlike Federated Learning \citep{pmlr-v54-mcmahan17a}, this so-called model merging \citep{Raffel_model_merging} approach passively operates on shared model weights without coordinated training rounds. As parameter-efficient fine-tuning methods gain popularity, combining fine-tuned modules, especially low-rank adapters (LoRAs) \citep{Hu_et_al_lora}, has emerged as a data-free alternative to enhance model capabilities \citep{beck-etal-2022-adapterhub, huang2024lorahubefficientcrosstaskgeneralization, zhao2024mergingloraslikeplaying, ostapenko24_arrow, prabhakar2024lorasoupsmergingloras, zhao-etal-2024-loraretriever, yadav2025a}. Users can exchange and merge LoRA updates at inference time like plug-and-play libraries. Consequently, this idea has sparked a proliferation of novel methods for reusing fine-tuned LoRA weights for new tasks \citep[e.g.][]{huang2024lorahubefficientcrosstaskgeneralization, ostapenko24_arrow, zhao2024mergingloraslikeplaying, beck-etal-2022-adapterhub, zhang2025lori}.

These approaches are appreciated for their computational and economic efficiency. However, they are often developed and validated under varying experimental conditions, with differing assumptions about system architecture, data availability, usage scenarios, and computational budgets. While recent work has addressed such inconsistencies in combining entire fine-tuned foundation models \citep{tam2024realisticevaluationmodelmerging}, the various design choices for merging or routing LoRA modules have only been surveyed \citep{yadav2025a}, leaving many questions unanswered. Furthermore, Large Language Models (LLMs) gain knowledge through pretraining, while supervised fine-tuning of instruction-following tasks teaches them the style or format for user interaction \citep{NEURIPS2023_ac662d74}. Consequently, fine-tuning LLMs with new knowledge often leads to hallucinations \citep{gekhman2024doesfinetuningllmsnew, pmlr-v235-ghosal24a}. Low-Rank Adaptations (LoRAs) are inherently limited in their expressiveness \citep{zeng2024expressivepowerlowrankadaptation} and can reduce chain-of-thought (CoT) reasoning abilities \citep{lobo2025impactfinetuningchainofthoughtreasoning}. 

The effectiveness of combining LoRA modules to generalize to new tasks is a critical concern. 
This position paper presents theoretical analysis and empirical findings on synthetic reasoning tasks to demonstrate the limitations of merging or routing LoRAs for zero-shot generalization to unseen tasks. 
Our findings indicate that combining LoRAs is ineffective for new tasks unless those tasks are already represented in the fine-tuning datasets.
Low-level statistics, such as the familiarity of entities or Chain-of-Thought templates, serve as crucial bridges for integrating disjoint information among LoRAs to generate logical answers to novel queries. Understanding these mechanisms is crucial for designing systems that can effectively reuse LoRAs and create suitable fine-tuning datasets. Designers must consider the specific applications for LoRA reuse, as curated training data is vital for successful combination. \textbf{
Our position hence is: We advocate for a shift in focus from algorithmic innovation to a rigorous understanding of the boundaries of adapter-based merging or routing, leveraging synthetic data and theoretical analysis.}

In the following sections, we begin with a discussion of related work and some overlooked perspectives. We then present theoretical analysis and empirical results that reveal the limitations of combining LoRAs, using synthetic two-hop reasoning and math problem setups. After discussing some alternative views, we conclude with our position on the effectiveness of LoRA combination.

\section{Discussion on Related Work Perspectives}
\label{sec:related}

LoRA modules \citep{Hu_et_al_lora} have emerged as a privacy-friendly, data-free method for sharing model capabilities, allowing users to exchange LoRA updates and merge them at inference time like plug-and-play libraries \citep{beck-etal-2022-adapterhub, huang2024lorahubefficientcrosstaskgeneralization, zhao2024mergingloraslikeplaying, ostapenko24_arrow, prabhakar2024lorasoupsmergingloras, zhao-etal-2024-loraretriever, yadav2025a}. However, many recycling methods require examples from unseen tasks to estimate merging weights or routers, raising questions about how much successful generalization can be attributed to the LoRAs themselves. This highlights the need for mechanisms ensuring effective LoRA combination under limited data access.
Weight averaging is a popular method for recycling LoRAs, inspired by findings that fine-tuned models remain in the same loss basin as pretrained weights \citep{NEURIPS2020_Neyshabur}. Task vectors, extracted from the difference between pretrained and fine-tuned model weights, can steer model behavior through arithmetic operations \citep{ilharco2023editingmodelstaskarithmetic}. Recent algorithms focus on resolving merge conflicts \citep{yadav2023tiesmerging}, randomly pruning redundant parameters \citep{pmlr-v235-yu24p}, and estimating weights for averaging LoRAs \citep{huang2024lorahubefficientcrosstaskgeneralization, prabhakar2024lorasoupsmergingloras}, but the mechanisms enabling successful generalization remain unexplored.
Mixture of Experts (MoE) architecture leverages fine-tuned LoRA adapters for novel domains. However, many methods require the domain data to setup the MoE for retrieving experts or training routers \citep{chronopoulou2023adaptersoupweightaveragingimprove, zhao-etal-2024-loraretriever, Jang_ICML_23}. Arrow \citep{ostapenko24_arrow} is a notable exception, which routes LoRAs directly based on similarity between query tokens and singular values of LoRA experts. 

\textbf{What can be recycled from a hub of LoRAs?}
When privacy is crucial, options are limited, especially for zero-shot generalization without training data. The latent logic in pretraining corpus or term frequency may play roles in combining LoRAs for zero-shot generalization. Scaling language models has shown emergent abilities for zero-shot reasoning \citep{wei2022emergent, NEURIPS2022_8bb0d291}, suggesting latent logical knowledge acquisition. Task vectors demonstrate analogical reasoning through arithmetic operations, but their effectiveness may depend on term co-occurrence frequency in pretraining data \citep{merullo2025on}. If the low-rank adapters are linear approximation of the fine-tuned tasks, such term-frequency effect in pretraining may set the limit for the combination of LoRAs to generalize to tasks that are underrepresented in the pretraining dataset. An alternative is that observed generalization performance via merging or routing LoRAs reflects superficial pattern-matching rather than genuine compositionality. Empirical studies indicate LLMs rely on token-level cues, with small lexical changes affecting reasoning performance \citep{mirzadeh2024gsmsymbolicunderstandinglimitationsmathematical, li2024gsmpluscomprehensivebenchmarkevaluating}. LLMs struggle with latent multi-hop reasoning, relying on explicit prompting to bridge compositionality gaps \citep{press-etal-2023-measuring, balesni2025twohopcursellmstrained}. Synthetic reasoning tasks thus play a key role in assessing compositional generalization, which indicates how effectively LoRA combination transfers to entirely novel tasks.

We began with theoretical analysis in 2-hop reasoning scenarios, mimicking real-world cases where models answer questions about unseen entity relationships. Using synthetic data to avoid pretraining contamination, we tested whether combining LoRAs without further training enables zero-shot solutions for 2-hop reasoning and complex math problems. Finally, we repeated experiments on models pretrained with different methods (e.g., chain-of-thought distillation, math corpus) to identify conditions for successful LoRA reuse.

\section{Theoretical Analysis}

Here, we argue theoretically that low-rank adaptation, while it can  store new facts in transformers, is  unlikely to lead to compositional behavior when combining different LoRAs.
We study this by considering the problem of composing knowledge from two LoRAs, where each contains factual knowledge, and their combination is expected to perform two-hop reasoning \citep[e.g.][]{yang2024large, balesni2025twohopcursellmstrained} that requires both pieces of knowledge.
In general, direct theoretical understanding of multi-layer softmax transformers is very difficult; but many theoretical insights have been obtained by studying one-layer models and the limit of very wide models.
We use this approach to perform a simple analysis of low-rank adaptation for factual knowledge.
Our setup is inspired by a prior theoretical study of factual recall in transformers \citep{nichani2025understanding} focusing on one-layer transformers.

For simplicity, we focus on the special case of a single attention head, and do not assume noise tokens in the context.
\citet[][Section 4]{nichani2025understanding} show that either an MLP or an attention head can perform factual recall.
Adapting the MLP with LoRA on a one-hop prompt can change individual facts -- such as setting $r_1(x_1)$ to $x_2$. However, importantly, combining LoRAs adapting two relations will not result in compositional behavior, as we show below.

\subsection{Setup}
\paragraph{Entities, Facts, and Prompts}
We consider a set $\mathcal{X}$ of entities, and a set $\mathcal{R}$ of binary relations $r \subset \mathcal{X} \times \mathcal{X}$ (e.g., X is married to Y; Y lives in Z, etc).
We assume that each $r$ is a partial function, i.e., for each $x$, there is at most one $y$ satisfying $(x,y) \in r$; we write $y = r(x)$.
Whereas \citet{nichani2025understanding} relied on the assumption that each relation maps to a disjoint output space, we avoid this assumption.
We assume that the model operates on the following prompts

\begin{enumerate}
\item One-Hop: \texttt{X REL} (where \texttt{X} represents an entity $x \in \mathcal{X}$ and \texttt{REL} represents a relation $r \in \mathcal{R}$, with expected completion: \texttt{Y}, where $y = r(x)$ (e.g., ``the spouse of X is Y'').
\item Two-Hop: \texttt{X REL1 REL2} (where \texttt{REL1}, \texttt{REL2} represent relations $r_1, r_2 \in \mathcal{R}$), with expected completion: \texttt{Y}, where $y = r_2(r_1(x))$ (e.g. ``the place of birth of the spouse of X is Y'').

\end{enumerate}

\paragraph{Simple Transformer Model}
We consider a vocabulary consisting of relations $r_1, r_2, \dots$ and entities $x_1, x_2, \dots$; with token embeddings $e_{r_i}, e_{x_i} \in \mathbb{R}^d$. We will write $E \in \mathbb{R}^{\left|\mathcal{X} \cup \mathcal{R}\right| \times d}$ for the matrix holding all token embeddings. 
We assume a single softmax attention head with $K, Q, V \in \mathbb{R}^{d\times d}$ matrices, and a ReLU MLP with hidden dimension $m$ given by matrices $U \in \mathbb{R}^{m \times d}$; $W \in \mathbb{R}^{|\mathcal{X}| \times m}$, mapping a vector $x$ to $W \cdot ReLU(U x )$.
We do not require positional encodings.
We assume that the next-token prediction is provided by $W$ as a one-hot encoding of the target entity, $i_x$, omitting softmax for simplicity.
Our aim is to showcase limitations in composition, not in storage of knowledge itself; hence, we allow the model a width $d$ substantially larger than $|\mathcal{X}| |\mathcal{R}|$. In order to give the MLP as much capacity as needed, we allow $m$ to be arbitrarily large.

\citet{nichani2025understanding} took $E$ to be randomly initialized and not traineed. We follow this assumption, and  additionally take  $U, V$ to remain untrained, as we do not assume noise tokens in the context.
Overall, we assume that $U, V, E$ matrices  are randomly initialized, all with entries from $\mathcal{N}(0, \frac{1}{\sqrt{d}})$.
We focus consideration of training to $W$.
This represents a random features setup \citep[e.g.][]{rahimi2008weighted, ghosal2022randomly, dirksen2022separation}. 
In this setup, softmax attention is close to uniform; we will take it to be exactly uniform for simplicity.
We will examine the situation where the base model already performs correctly for the given set of relations $\mathcal{R}$, and $W$ is then adapted to reflect edits to such facts.

We focus LoRA on $W$, in agreement with our experimental finding in Section~\ref{sec:experiment-2hop} that applying to MLPs can be sufficient to get most of the gains.
We consider updates $\Delta W = AB^T$ with $A \in \mathbb{R}^{|\mathcal{X}| \times s}, B \in \mathbb{R}^{m \times s}$ where $s$ is small,  subject to an L2 penalty $\|A\|_{F} + \|B\|_F$.
We particularly consider one-rank updates, $\Delta W = pq^T$ where $p \in \mathbb{R}^{|\mathcal{X}|}, q \in \mathbb{R}^m$.

We note that this setup simplifies many aspects of transformers: there is only one layer and one head, and training focuses on the (linearized) output. We also remove the softmax over the vocabulary in the output.
Our setup is designed to be \emph{simplest possible} setup in which a nontrivial statement about LoRA's ability to learn and combine abilities can be made.

\subsection{Results}
The correct responses to all 1-hop and 2-hop relations can jointly be coded into $W$ when $d$ and $m$ are sufficiently large, due to the separation ability of the random features model \citep{ghosal2022randomly}. 
This analysis is in line with mechanistic studies of factual recall suggesting MLPs act as key-value storage \citep{geva2021transformer}.

Changing a fact $y = r(x)$ requires changing the output of the MLP on the subspace spanned by the entity and relation. 
When the update affects only a single fact, L2 regularization ensures that it has a simple and interpretable closed form:

\begin{proposition}\label{prop:update}
A rank-one update to $W$ changing the output on a prompt \texttt{X REL} from $r(x)$ to $\tilde{r}(x)$ must have the form: 
\begin{equation}
\Delta W_{r \mapsto \tilde{r}} = \frac{1}{\|ReLU(UV e_{\texttt{X}} + U e_{\texttt{REL}} )\|_2^2}  (i_{\tilde{r}(x)} - i_{r(x)}) ReLU(U\cdot V e_{\texttt{X}} + U\cdot e_{\texttt{REL}} )^T
\end{equation}
\end{proposition}
This is similar to the  RoME update \citep{meng2022locating}.
Intuitively, based on the idea that MLPs act as key-value storage, the LoRA update $\Delta W = AB^T$ specifically addresses the encoding of the prompt \texttt{X REL} in the $B$ matrix, and the changed output in the $A$ matrix.
The proof is in Appendix~\ref{sec:theory-appendix}.

Now consider a two-hop prompt \texttt{X REL1 REL2}, intended to denote the composition of the two relations.
Given sufficient width, any set of such two-hop facts can be encoded in $W$.
However, as we next show, adding two LoRAs modifying two relations ($\Delta W_{r\mapsto \tilde{r}}, \Delta W_{r\mapsto \hat{r}}$) will not unlock compositional behavior on the new facts:
\begin{thm}\label{thm:two-hop-difficult}
    Assume LoRAs $\Delta W_{r_1\mapsto \tilde{r}_1}, \Delta W_{r_2\mapsto \hat{r}_2}$ are created to adapt two single facts for $r_1$, $r_2$.
    Summing these adapters will not result in correct results for composition of the two relations $r_1, r_2$. 
\end{thm}
The formal proof is in Appendix~\ref{sec:theory-appendix}.
The reasoning is as follows. As shown in Proposition~\ref{prop:update}, the two LoRAs specifically modify the MLP output on the subspaces inhabited by the activations computed on the two one-hop prompts.
When the model encounters a two-hop prompt, the activations will partly overlap with the subspaces for both one-hop prompts, and the adapters will lead the model to output not the composition $r_2(r_1(x))$, but a linear mixture of two relevant entities.
A natural question is whether some of the routing or weighting methods proposed in the literature resolve this; it turns out that the argument extends to those: For instance, weighted averaging of the two adapters \citep[e.g.][]{prabhakar2024lorasoupsmergingloras, ostapenko24_arrow} will still fail to perform compositionally when several facts are updated (see Appendix~\ref{sec:theory-discussion} for more).
Yet another approach might be to combine a larger library of LoRAs where some have been trained on 2-hop examples from other task pairs. One might hope that this would prime the model towards compositional behavior; however, the reasoning above still applies, and suggests that reusing LoRAs would still fail to behave compositionally (Appendix~\ref{sec:theory-discussion}). 

One limitation of our theoretical analysis is that (in line with \citet{nichani2025understanding}) it applies to a single-layer transformer; our experiments test applicability of the conclusions to LLMs across scales.

\section{Experiments}
\label{sec:experiments}

Our experiments aim to test under what circumstances combining LoRAs can enable LLMs to perform new tasks that require logical combinations of different LoRA's expertise. Our primary focus is on the two data-agnostic routing methods, \emph{Uniform} averaging and \emph{Arrow} routing \citep{ostapenko24_arrow}, which directly work on shared LoRA experts' weights (see Appendix~\ref{app:twohop-result} for details). We synthesized two reasoning tasks, 2-hop reasoning and easy-to-hard math word problems, to examine the successful factors underlying their zero-shot generalization via reusing existing LoRAs on novel tasks. Specifically, we investigate the preconditions necessary for LoRA routing to be effective, such as entity familiarity, domain-specific pretraining of base models, and the necessity of the presence of novel tasks in fine-tuning LoRA experts. We assess how these effects on different routing strategies would scale with base model sizes, ranging from 3 billion to 70 billion parameters. Our findings emphasize the importance of domain-specific pretraining, common templates for LoRA fine-tuning datasets, and potential interference from routing compared to individual LoRA experts.

\subsection{Two-Hop generalization}\label{sec:experiment-2hop}

We investigate whether combining two LoRAs enables compositional reasoning. Building on our theoretical analysis and inspired by \cite{balesni2025twohopcursellmstrained}, we design a two-hop reasoning task requiring composition across linguistic variations while controlling for base model knowledge. The dataset uses a fixed structure: \textbf{First Hop} ($A \to B$, identifying the spouse of a given entity $A$), \textbf{Second Hop} ($B \to C$, identifying the residence of $B$), with the goal of inferring $A \to C$. This setup closely follows our Theorem~\ref{thm:two-hop-difficult}, which suggests that LoRAs trained on one of the two hops each would, if combined, not unlock the indirect relationship.

We conduct three datasets varying the nature of entity names and locations while ensuring that the relational facts remain synthetic: $F$ (fake names, fake locations), where both entities and locations are synthetic (e.g., $(Zint, Frosk, Narik)$); $H$ (fake names, real locations), where names are synthetic but locations are real (e.g., $(Zint, Frosk, London)$); and $R$ (real names, real locations), where both names and locations are real, but relationships are deliberately shuffled to remain false (e.g., $(Barack\ Obama, Camila\ Alves, London)$). We refer to the first-hop ($A\to B$), second-hop ($B\to C$), and the two-hop ($A\to C$) subsets of each dataset as $F_1, F_2, F_{12}$, $H_1, H_2, H_{12}$, and $R_1, R_2, R_{12}$, respectively (see Table~\ref{tab:notation-examples} and Appendix~\ref{sec:2hop-expt-appendix} for examples and details).

Based on ablation studies (Table~\ref{tab:two-hop-finetuned-layers} and ~\ref{tab:bridge_dataset_performance} in Appendix ~\ref{app:twohop-result}), we focused on fine-tuning only the MLP layers of the following base models: Qwen2.5-3B-Instruct, Qwen2.5-7B-Instruct, and Qwen2.5-14B-Instruct \cite{qwen2025qwen25technicalreport}, as well as DeepSeek-R1-Distill-Qwen-7B, DeepSeek-R1-Distill-Qwen-14B, and DeepSeek-R1-Distill-LLaMA-70B \cite{deepseekai2025deepseekr1incentivizingreasoningcapability}.

\subsubsection{Impact of base model and familiarity}
\paragraph{Experiment setup}
For each dataset ($F$, $H$, $R$), we train four LoRA adapters (experts) LoRA 1, LoRA 2, LoRA 3 (Oracle Expert), and LoRA 4 (Mixed Two-Hop Expert) on relation $A \rightarrow B$, $B \rightarrow C$, $A \rightarrow C$, and mixed data ($A \rightarrow B$, $B \rightarrow C$), respectively. From these experts, we construct two libraries: \textbf{2-combination library} which includes LoRA 1 and LoRA 2; and \textbf{3-combination library} which includes LoRA 1, LoRA 2, and LoRA 3. We evaluate the model’s ability to generalize and infer $A \rightarrow C$ relationships. We use Chain-of-Thought (CoT) prompting during testing for the 3-combination library and the 2-combination library.
\paragraph{Results and analysis}
As shown in Figure~\ref{fig:lora-result}a, performance on the 3-combination library for the $H$ dataset improves with model size, while the 2-combination library remains consistently poor. Figure~\ref{fig:lora-result}b shows that $R$ outperforms $H$, both far exceeding $F$. Notably, with only $A \rightarrow B$ and $B \rightarrow C$ adapters (2-combination library), accuracy stays below 10\%, supporting Theorem~\ref{thm:two-hop-difficult} that composing knowledge across separate LoRAs is inherently challenging. Even when $A \rightarrow C$ is covered (3-combination library), routing does not always succeed, in particular in smaller models and the presence of  of unfamiliar entities, such as fake names or cities in the $F$ dataset. These trends hold across datasets and model families (see Table~\ref{tab:lora_results} in Appendix~\ref{app:twohop-result}).
\begin{figure}[ht]
    \centering
    \includegraphics[width=1\linewidth]{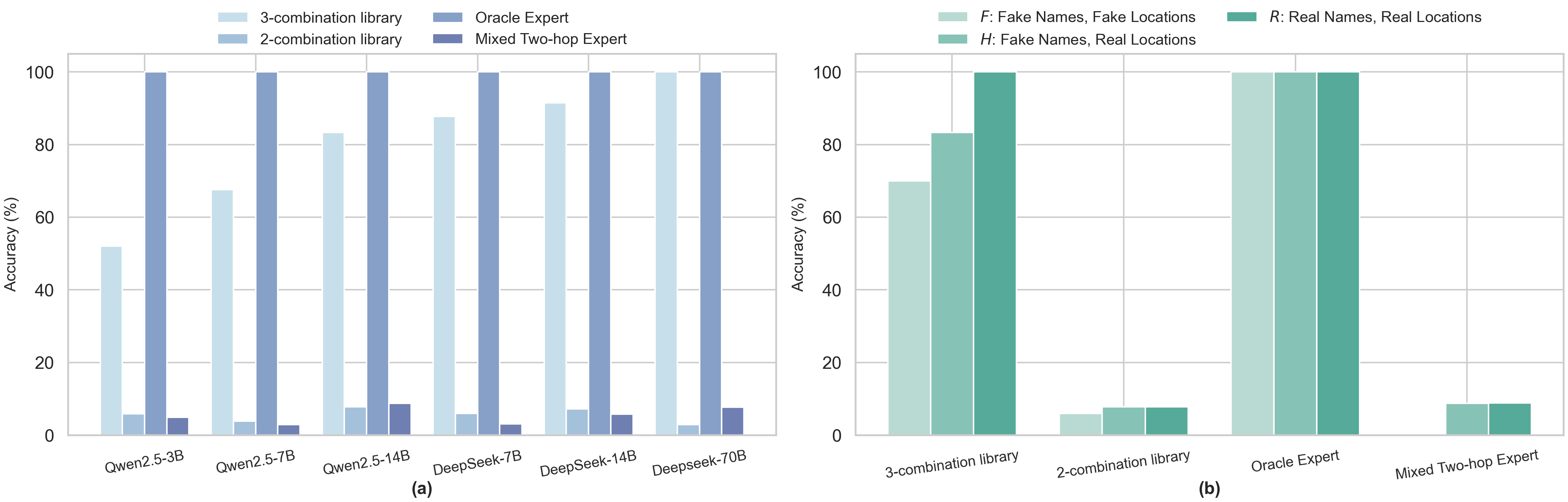}
   \caption{Performance of LoRA libraries and individual experts on two-hop datasets. (a) Comparison of library-level and expert-level performance on the test set of $H$ across different base models. (b) Impact of entity familiarity on the performance of LoRA combination methods, using Qwen2.5-14B-Instruct as the base model, evaluated on three test sets: $F$, $H$, and $R$. Across setups and model sizes, performance on the 2-combination library is poor; including an expert trained on the target task ($A \rightarrow C$ relationship)  is necessary.}
    \label{fig:lora-result}
\end{figure}

\subsubsection{Composition requires close match between testing and training prompts} 
So far, we found that synthesizing two-hop reasoning from two LoRAs is difficult, in agreement with our theoretical predictions.
What strategies could enable composition? 
While a substantial amount of work has found chain-of-thoughts to support reasoning (including in the 2-hop setup when employing full-parameter finetuning, \citep{balesni2025twohopcursellmstrained}), we found that composition from the two $A\rightarrow B$ and $B\rightarrow C$ experts performs poorly even with chain-of-thought prompting.
Our theoretical analysis suggests that, as LoRA adapters rely on targeting specific low-dimensional subspaces, compositional behavior can only be unlocked when the target prompts show a close formal match to prompts on which the LoRAs were trained.
CoT prompting might thus be insufficient as the form of the targets mismatches the one-hop training examples of the LoRAs.
However, mixing in CoT examples into the LoRA training data might be sufficient.
In this section, we test if it is possible to enable composition by including CoT templates in the training data, and how close the match in reasoning patterns needs to be between finetuning and testing datasets. We specifically define the following \emph{bridge} technique, and designed a series of experiments to determine how closely to target task needs to be present in the LoRA finetuning datasets to enable compositional behavior.
The idea is that the finetuning dataset additionally includes examples of the targeted reasoning pattern.
We test, via ablations, which aspects of the targeted reasoning pattern needs to be present. 
\paragraph{Experimental setup}
We design two \emph{bridge} variants over the $F$ and $R$ datasets. The \textbf{Fake Bridge} ($B_F$) is constructed by concatenating the $F_1$ ($A \to B$), $F_2$ ($B \to C$), and $F_{12}$ CoT (CoT-formatted $A \to C$) subsets. The \textbf{Real Bridge} ($B_R$) follows the same structure but uses the corresponding subsets ($R_1$, $R_2$, $R_{12}$ CoT) from the $R$ dataset, which contains real names and locations (see Table~\ref{tab:bridge-dataset} and Figure~\ref{fig:2-hop} in the Appendix~\ref{sec:2hop-expt-appendix} for examples and details). We fine-tune adapters on the union of direct-hop examples and a bridge dataset. 
In the first configuration, \textbf{\textit{Setup 1}}, models are trained by mixing fake data subsets ($F_i$) with the Real Bridge set ($B_R$): LoRA 1 is trained on $F_1 + B_R$, and LoRA 2 on $F_2 + B_R$, with evaluation on the held-out subset $F_{12}$. In the second configuration, \textbf{\textit{Setup 2}}, models are trained by mixing real data subsets ($R_i$) with the Fake Bridge set ($B_F$): LoRA 1 is trained on $R_1 + B_F$, and LoRA 2 on $R_2 + B_F$, with evaluation on the held-out subset $R_{12}$. 
\begin{figure}[ht]
    \centering
    \includegraphics[width=0.92\linewidth]{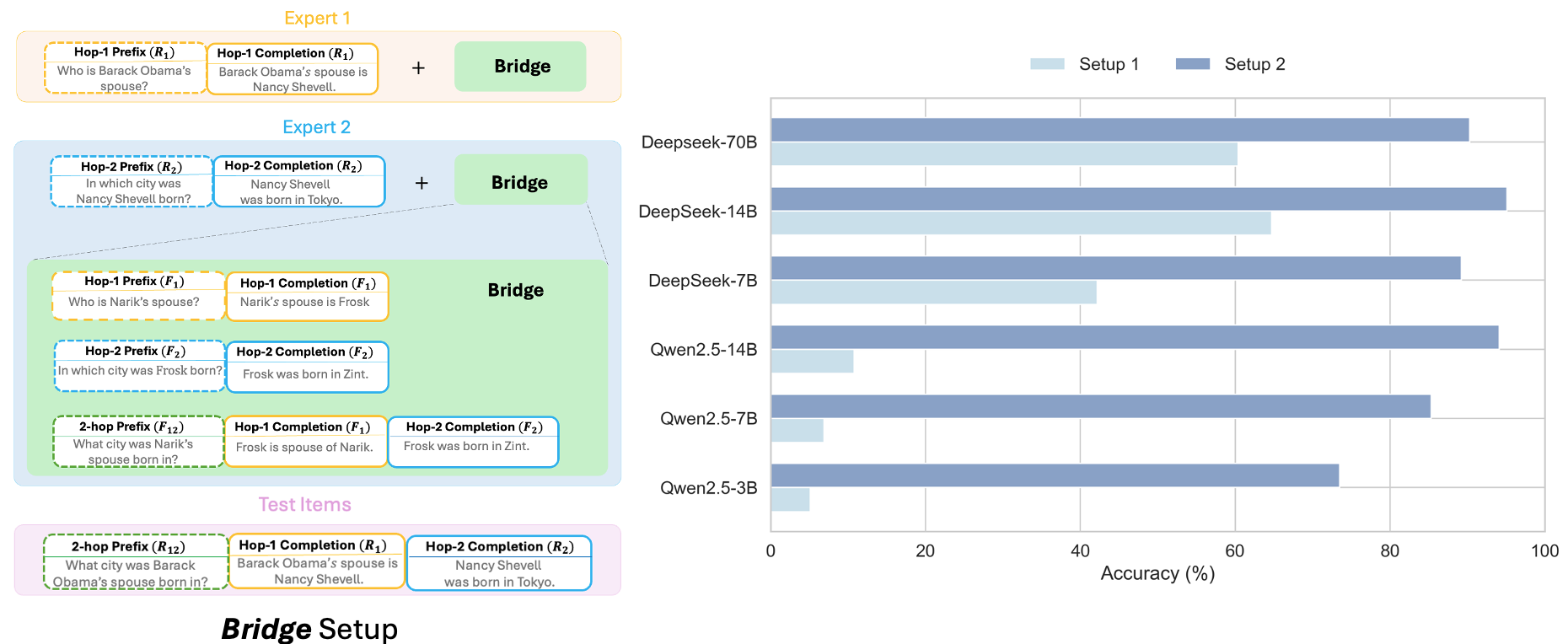}
    \caption{
    (left) In the bridge setup, both LoRA experts are trained not only on one of the two hops, but also on examples (disjoint from those needed in testing) of both hops, and chain-of-thought two-hop reasoning. 
    (right) Performance comparison of two setups across different base models: Setup 1 (\textbf{Real Bridge}, $B_R$) adds a bridge using real names and locations  to a dataset using fake names and fake locations ($F$), while Setup 2 (\textbf{Fake Bridge}, $B_F$) reverses these. Each setup uses LoRA 1 and LoRA 2 as the experts in the library. The bridge setup is much more successful in Setup 2.}
    \label{fig:bride-result}
\end{figure}

\paragraph{Results and analysis}
Figure~\ref{fig:bride-result} demonstrates that explicitly incorporating the target two-hop reasoning pattern into the LoRA fine-tuning data is crucial for achieving reliable compositional generalization. In \textbf{\textit{Setups 1}} and \textbf{\textit{2}}, where each adapter is trained on a synthetic direct-hop subset combined with the bridge dataset, \emph{Arrow} performance improves significantly compared to earlier experiments. Additionally, the bridge setup is much more successful in \textbf{\textit{Setup 2}}, highlighting the importance of entity familiarity for effective generalization.

We conducted a set of ablations to analyze which aspects are important for the success of this strategy.
Incorporating structured reasoning into LoRA finetuning yields only marginal gains unless the finetuning data closely mirror the target two-hop task.
First, as shown in Table~\ref{tab:bridge_ablation}, the bridge 
improves two-hop accuracy only when CoT-formatted $A \rightarrow C$ instances are included during adapter training. Omitting CoT formatting results in worse performance (\textbf{\textit{Setups 2}} vs. \textbf{\textit{3}}).
Second, simply including the bridge in only one of the two LoRA adapters (\textbf{\textit{Setups 4, 5}}), or providing only $A \rightarrow C$ prompts without the individual one-hop tasks (\textbf{\textit{Setups 6, 7}}), results in significantly weaker compositional performance compared to setups where both $F_1$ ($A \rightarrow B$) and $F_2$ ($B \rightarrow C$) are included alongside CoT bridging. This highlights the importance of exposing the model to each subtask.
Third, we found that  relaxing the bridge to use disjoint task pairs still produces nontrivial gains (\textbf{\textit{Setup 8}}, which uses completely different relations as the bridge dataset: $F_4$: \textit{study\_in} and $F_5$: \textit{child\_of}, see Table~\ref{tab:notation-examples} in the Appendix~\ref{sec:2hop-expt-appendix} for examples), suggesting that exact task-pair matching is less critical so long as the finetuning set contains examples reflecting the overall reasoning pattern. Altogether, these results confirm the importance of CoT exemplars and the individual tasks to unlocking generalization of Arrow routing, even if the examplars are semantically different from the target task.

Aside from combinations of two or three LoRAs, we further tested what happens when increasing the number of tasks present in the collection of LoRAs, including both various one-hop tasks, and also various two-hop tasks. Even in this case, we found that composition was very difficult, again in agreement with our theoretical predictions (Analysis in Appendix ~\ref{app:larger_set}). 

Overall, these results support our conclusion that (i) direct composition of knowledge from different LoRAs is very difficult, (ii) the LoRA finetuning datasets must contain examples closely matching the target reasoning behavior.

\begin{table}[htbp]
\caption{Ablations for the bridge training setups. We compare the 2-combination libraries trained on just the two hops (0), the full bridge (2), with various strategies interpolating between these, such as providing a bridge in only one expert (4 and 5), or omitting the CoT template  from the bridge (3). The full bridge attains highest performance, and versions not including a bridge CoT in both of the two experts show poor performance (0, 3, 4, 5).}
\label{tab:bridge_ablation}
\centering
\begin{adjustbox}{max width=\textwidth}
\begin{tabular}{cllcccccc}
\hline
\textbf{Setup} & \textbf{LoRA 1} & \textbf{LoRA 2} & \textbf{\begin{tabular}[c]{@{}c@{}}Qwen2.5-\\ 3B-Instruct\end{tabular}} & \textbf{\begin{tabular}[c]{@{}c@{}}Qwen2.5-\\ 7B-Instruct\end{tabular}} & \textbf{\begin{tabular}[c]{@{}c@{}}Qwen2.5-\\ 14B-Instruct\end{tabular}} & \textbf{\begin{tabular}[c]{@{}c@{}}DeepSeek-R1-\\ Distill-Qwen-7B\end{tabular}} & \textbf{\begin{tabular}[c]{@{}c@{}}DeepSeek-R1-\\ Distill-Qwen-14B\end{tabular}} & \textbf{\begin{tabular}[c]{@{}c@{}}DeepSeek-R1-\\ Distill-Llama-70B\end{tabular}} \\ \hline
\multirow{2}{*}{0} & \multirow{2}{*}{\textbf{$R_1$}} & \multirow{2}{*}{\textbf{$R_2$}} & \multirow{2}{*}{1\%} & \multirow{2}{*}{3.9\%} & \multirow{2}{*}{7.8\%} & \multirow{2}{*}{8.8\%} & \multirow{2}{*}{3.9\%} & \multirow{2}{*}{5.9\%} \\
 &  &  &  &  &  &  &  &  \\ \hline
\multirow{2}{*}{2} & \textbf{$R_1$} & \textbf{$R_2$} & \multirow{2}{*}{\textbf{73.5\%}} & \multirow{2}{*}{\textbf{85.3\%}} & \multirow{2}{*}{\textbf{94.1\%}} & \multirow{2}{*}{\textbf{89.2\%}} & \multirow{2}{*}{\textbf{95.1\%}} & \multirow{2}{*}{\textbf{90.3\%}} \\
 & $F_1 + F_2 + F_{12}$ CoT ($B_F$) & $F_1 + F_2 + F_{12}$ CoT ($B_F$) &  &  &  &  &  &  \\ \hline
\multirow{2}{*}{3} & \textbf{$R_1$} & \textbf{$R_2$} & \multirow{2}{*}{1\%} & \multirow{2}{*}{2.9\%} & \multirow{2}{*}{3.9\%} & \multirow{2}{*}{3.9\%} & \multirow{2}{*}{11.8\%} & \multirow{2}{*}{3.9\%} \\
 & $F_1 + F_2 + F_{12}$ & $F_1 + F_2 + F_{12}$ &  &  &  &  &  &  \\ \hline
\multirow{2}{*}{4} & \textbf{$R_1$} & \multirow{2}{*}{\textbf{$R_2$}} & \multirow{2}{*}{9.8\%} & \multirow{2}{*}{8.8\%} & \multirow{2}{*}{11.8\%} & \multirow{2}{*}{8.8\%} & \multirow{2}{*}{12.7\%} & \multirow{2}{*}{10.2\%} \\
 & $F_1 + F_2 + F_{12}$ CoT ($B_F$) &  &  &  &  &  &  &  \\ \hline
\multirow{2}{*}{5} & \multirow{2}{*}{\textbf{$R_1$}} & \textbf{$R_2$} & \multirow{2}{*}{10.8\%} & \multirow{2}{*}{35.3\%} & \multirow{2}{*}{27.5\%} & \multirow{2}{*}{25.9\%} & \multirow{2}{*}{22.5\%} & \multirow{2}{*}{26.1\%} \\
 &  & $F_1 + F_2 + F_{12}$ CoT ($B_F$) &  &  &  &  &  &  \\ \hline
\multirow{2}{*}{6} & \textbf{$R_1$} & \textbf{$R_2$} & \multirow{2}{*}{40.2\%} & \multirow{2}{*}{61.8\%} & \multirow{2}{*}{65.7\%} & \multirow{2}{*}{75.5\%} & \multirow{2}{*}{72.5\%} & \multirow{2}{*}{87.3\%} \\
 & $F_{12}$ CoT & $F_{12}$ CoT &  &  &  &  &  &  \\ \hline
\multirow{2}{*}{7} & \textbf{$R_1$} & \textbf{$R_2$} & \multirow{2}{*}{27.5\%} & \multirow{2}{*}{75.5\%} & \multirow{2}{*}{79.4\%} & \multirow{2}{*}{74.5\%} & \multirow{2}{*}{76.9\%} & \multirow{2}{*}{89.9\%} \\
 & $F_1 + F_{12}$ CoT & $F_2 + F_{12}$ CoT &  &  &  &  &  &  \\ \hline
\multirow{2}{*}{8} & \textbf{$R_1$} & \textbf{$R_2$} & \multirow{2}{*}{43.1\%} & \multirow{2}{*}{82.4\%} & \multirow{2}{*}{92.2\%} & \multirow{2}{*}{81.6\%} & \multirow{2}{*}{83.5\%} & \multirow{2}{*}{87.9\%} \\
 & $F_1 + F_4 + F_{14}$ CoT & $F_5 + F_2 + F_{52}$ CoT &  &  &  &  &  &  \\ \hline
\end{tabular}
\end{adjustbox}
\end{table}

\subsection{Generalization from Easy to Hard Math Word Problems}

To evaluate whether our findings hold in more realistic settings, we use the GSM-Symbolic benchmark \citep{mirzadeh2024gsmsymbolicunderstandinglimitationsmathematical}, which enables controlled assessment of reasoning robustness in math across well-defined difficulty levels. Each LoRA expert was fine-tuned on GSM-Symbolic (original) and GSM-P1 (with one added clause) individually, before being combined for evaluation on GSM-P2 (which adds another clause). We compare general-purpose and math-specialized models to assess the impact of pretraining. Similar to exposing LoRAs to solutions closely resembling the target task, we also tested whether fine-tuning with reusable Markdown and Python code \citep{suzgun2025dynamiccheatsheettesttimelearning} would improve generalization on GSM-P2. Detailed experimental design, fine-tuning, and evaluation procedures can be found in Appendix Section ~\ref{sec:gsm-expt}.

\paragraph{Limitations of LoRA Routing for Compositional Generalization.}
The effectiveness of LoRA routing is highly dependent on the base model’s pretraining history. To start, we replicated the findings of \citet{mirzadeh2024gsmsymbolicunderstandinglimitationsmathematical}, which show that large language models (LLMs) lack robustness in mathematical reasoning (see Appendix Section~\ref{sec:fragility-reasoning-appendix}, Table~\ref{tab:gsm8k-gsm-symbolic-benchmarks}). Routing methods such as Uniform and Arrow provided modest improvements for the general-purpose Qwen2.5-1.5B-Instruction model, but often degraded performance in math-specialized models like Qwen2.5-Math-Instruction, regardless of model size (Table~\ref{tab:gsm_lora_merge_performance}). Among these, Uniform consistently outperformed Arrow. Echoing prior work showing that 8-shot GSM8K in-context examples do not improve performance on GSM-P2 \citep{mirzadeh2024gsmsymbolicunderstandinglimitationsmathematical}, we further observed that combining these examples with LoRA routing actually worsened results. For example, in the Qwen2.5-Math-7B-Instruction model, Arrow routing with in-context examples reduced GSM-P2 accuracy from 0.27 to 0.06 (see Appendix Section~\ref{sec:in-context-learning}, Table~\ref{tab:gsm_lora_merge_performance_appendix} for details).

The performance drop observed after LoRA routing may stem from a mismatch between the fine-tuning data and the base model’s capabilities. Qwen2.5-Math-Instruction is designed to solve problems using Markdown and Python code, while the GSM-Symbolic benchmarks provide only natural language Chain-of-Thought (CoT) solutions. As a result, routing LoRAs fine-tuned on this dataset may suppress the model's tool-integrated reasoning abilities and lead to an increase in calculation errors. Our error analysis follows the definitions and procedures outlined by \citet{zhong2025achieving97gsm8kdeeply}. See Appendix Section~\ref{sec:error-analysis} and Table~\ref{tab:gpt4o_as_judge_response_analysis} for details.

{\normalsize
\begin{table}[ht]
\caption{Accuracy comparison on zero-shot GSM-P2 after routing LoRA experts individually fine-tuned on GSM-Symbolic and GSM-P1.}
\label{tab:gsm_lora_merge_performance}
\centering
\begin{adjustbox}{max width=\textwidth}
\begin{tabular}{@{}lccc@{}}
\toprule
\textbf{LoRA Routing Methods} & \textbf{Qwen2.5-1.5B-Inst} & \textbf{Qwen2.5-Math-1.5B-Inst} & \textbf{Qwen2.5-Math-7B-Inst} \\
\midrule

Base model only & 5\% & 47\% & 68\% \\

Uniform & 10\% & 24\% & 34\% \\

Arrow & 9\% & 19\% & 27\% \\
\bottomrule
\end{tabular}
\end{adjustbox}
\end{table}
}

\paragraph{How can programming language bridge the generalization gap?}Our experimental design is motivated by recent findings, Dynamic Cheatsheet, which demonstrate that encouraging language models to retain and apply reusable intermediate solutions during inference significantly improves their performance on math problems \cite{suzgun2025dynamiccheatsheettesttimelearning}. We extend this idea using the GSM-Symbolic benchmark \citep{mirzadeh2024gsmsymbolicunderstandinglimitationsmathematical}, where generalization from easier to harder problem variants requires understanding the full computational graph (Appendix Figure~\ref{fig:Gsm-Symbolic-P1-P2-Comp-Graph}). In the previous setting, each LoRA is fine-tuned on partial solutions corresponding to subsets of reasoning steps (e.g., the black or orange subgraphs in Appendix Figure~\ref{fig:Gsm-Symbolic-P1-P2-Comp-Graph}). However, routing these LoRAs alone does not suffice to solve the more complex P2 variant, which involves the complete computational graph (blue subgraph in Appendix Figure~\ref{fig:Gsm-Symbolic-P1-P2-Comp-Graph}). We hypothesize that reusable Markdown and Python solutions can bridge partial representations and enhance compositional generalization through LoRA routing, and to test this, we implemented two agent-based actor-critic workflows \citep{wu2024autogen} to generate fine-tuning data (See Appendix~\ref{app:proc_synthesizing_reusable_codes} for implementation details). Table~\ref{tab:tool-integration} demonstrate modest improvements in solving the complex P2 problems via routing LoRAs fine-tuned with these reusable code solutions. Such improvement is clearer in smaller model (Qwen2.5-Math-1.5B-Instruction) when fine-tuning targeted the MLP layers. This finding emphasizes the need for system designers to understand how to effectively reuse LoRA experts to guide data generation, while also noting that reusing LoRAs is most effective when target tasks are clearly defined beforehand.

\normalsize
\begin{table}[ht]
\caption{Enhancing easy-to-hard generalization by leveraging Tool-Integrated Reasoning (TIR) prompt and fine-tuning with reusable code.}
\label{tab:tool-integration}
\centering
\begin{adjustbox}{max width=\textwidth}
\begin{tabular}{@{}llcccc@{}}
\toprule
\textbf{Base model} & \textbf{Fine-tuned Modules} & \textbf{Routing Methods} & \textbf{LoRA(GSM-Symb) and LoRA(GSM-Symb-P1)} & \textbf{Base Model Only} \\
\midrule

\multirow{4}{*}{Qwen2.5-Math-1.5B-Inst} & attention & Uniform & 14\% & \multirow{2}{*}{18\%} \\
&  & Arrow & 14\% & \\
& MLP & Uniform & \textbf{20\%} & \multirow{2}{*}{} \\
&  & Arrow & 13\% &  \\
\midrule

\multirow{4}{*}{Qwen2.5-Math-7B-Inst} & attention & Uniform & \textbf{40\%} & \multirow{2}{*}{47\%} \\
&  & Arrow & 24\%  & \\
& MLP & Uniform & 39\%  & \multirow{2}{*}{} \\
&  & Arrow & 37\% & \\
\bottomrule
\end{tabular}
\end{adjustbox}
\end{table}

\section{Alternative Views}
\label{sec:alternativeViews}
While our findings indicate that combining LoRAs is ineffective for new tasks unless those tasks are already represented in the fine-tuning datasets, alternative PEFT methods may offer better compositional results. For instance, LoRI \citep{zhang2025lori} addresses cross-task inference by combining random projections with task-specific masks, potentially enabling better adapter routing. However, positive results for compositional reasoning have not been reported, and our theoretical analysis suggests that it remains challenging. Similarly, LoRA Lego \citep{zhao2024mergingloraslikeplaying} formalizes low-rank updates as composed of independent units and clusters these into new adapters to reduce interference, though it has not been shown to enable compositional reasoning. Self-MoE \citep{kang2024self} constructs experts based on self-generated specialization training data and a trained router, but it remains underexplored to what extent this method can enable compositional combination of different abilities. FLiX \citep{sun2024improving} learns different low-rank updates for various task or dataset features, and CS-ReFT \citep{sun2024improving} learns orthonormal task-specific adaptors. Despite these innovations, none have demonstrated effective compositional combination of skills, as our theoretical analysis suggests inherent limitations. Another perspective is to train models specifically for generalization and composition, even if it requires data from the target task. 

Recent work \citep{prabhakar2024lorasoupsmergingloras} has proposed LoRA concatenation as an effective method for composing skills to solve challenging math word problems, such as those in GSM-Hard \citep{pmlr-v202-gao23f}. We recognize the significance of these findings, particularly their demonstration that decomposing skills into reusable LoRAs and estimating appropriate combination weights can enhance performance, provided that additional task-specific data and knowledge are available. However, our work takes a different perspective. Unlike GSM-Hard \citep{pmlr-v202-gao23f}, which primarily modifies numerical ranges while preserving the question format of the original GSM8K problems, GSM-Symbolic-P2 \citep{mirzadeh2024gsmsymbolicunderstandinglimitationsmathematical} presents more realistic and difficult compositional generalization challenges. It altered the question format and the structural complexity of math problems into an entirely unseen problem forms. Our theoretical analysis shows the limit (Appendix~\ref{sec:theory-discussion}) that is supported by empirical results that training showed little gains in a 2-hop reasoning setting (Appendix Table~\ref{tab:other-routings}). This suggests that the benefits of such approaches may not extend to more challenging generalization tasks like GSM-Symbolic. While skill composition remains important, our results highlight a key limitation of LoRA routing approaches as shown in our theoretical analyses and empirical findings: their effectiveness often depends on foreknowledge or training data of the downstream tasks, which may not be viable in practice.

\section{Conclusions}
\label{sec:conclusions}

Both our theoretical and empirical findings indicate that combining LoRAs is ineffective for new tasks unless those tasks are already represented in the fine-tuning datasets. Familiarity with entities and domains is crucial, but routing strategies often interfere with performance, especially in smaller models and unfamiliar contexts. Targeting MLP layers during fine-tuning may offer some advantages, but the direct composition of knowledge from different LoRAs remains problematic. \textbf{We hence reiterate our position to call for a shift in research focus: from algorithmic innovation to understanding the boundaries of adapter-based routing by leveraging synthetic data and theoretical analysis, as this will facilitate the development of effective systems to recycle LoRAs in real-world applications.}

\section*{Contributions}
\label{sec:contributions}

M.-Y.C. conceived the idea, conducted literature review, led the research direction, designed and conducted the GSM experiments. T.T.U.H. designed and conducted the 2-hop experiments. M.H. conducted the theoretical analysis and supervised 2-hop experiments. All authors discussed the results and contributed to the final manuscript.

% Make the contribution transparent --> keep it until review process conclude: 
% Footnote: M.-Y., Uyen, H.M. - equal contribution in alpha betic order. --> adding stars upon publication.
% Also statement of who has done what in the end when camera ready:
% MY: gsm, conceive and lead
% Uyen: 2-hop
% MH: theory

\bibliography{neurips_2025_main_ref}
\bibliographystyle{unsrtnat}

\appendix

\vfill
\pagebreak
\section{Technical Appendices and Supplementary Material}
This appendix offers supplementary details to support our position in the main text, including:

\begin{itemize}
    \item \textbf{A.1:} Formal theoretical analysis.
    \item \textbf{A.2:} Detailed experimental setups and results for 2-hop experiments.
    \item \textbf{A.3:} Evaluating Larger Sets of LoRAs.
    \item \textbf{A.4:} Detailed experimental setups and results for easy-to-hard math word problems. 
\end{itemize}

\subsection{Theoretical Analysis}\label{sec:theory-appendix}

We formalize training of a LoRA adapter $\Delta W = AB^T$ on a training prompt (e.g. \texttt{X REL Y}) as choosing an adapter for which the model with adapted parameters $W + \Delta W$ outputs the target entity \texttt{Y} after the prompt \texttt{X REL}, while (among all such interpolating adapters) minimizing the regularizer $\|A\|_F^2 + \|B\|_F^2$. For this definition, we find:

\begin{proposition}[Repeated from Proposition 1]
A rank-one update to $W$ changing the output on a prompt \texttt{X REL} from $r(x)$ to $\tilde{r}(x)$ must have the form: 
\begin{equation}\label{eq:rank-one-update}
\Delta W_{r \mapsto \tilde{r}} = \frac{1}{\|ReLU(UV e_{\texttt{X}} + U e_{\texttt{REL}} )\|_2^2}  (i_{\tilde{r}(x)} - i_{r(x)}) ReLU(U\cdot V e_{\texttt{X}} + U\cdot e_{\texttt{REL}} )^T
\end{equation}
\end{proposition}

\begin{proof}
We have the demand that
\begin{equation}
    (W+\Delta W_{r \mapsto \tilde{r}}) ReLU(U\cdot V e_{\texttt{X}} + U\cdot e_{\texttt{REL}} ) = i_{\tilde{r}(x)}
\end{equation}
hence
\begin{equation}
    \Delta W_{r \mapsto \tilde{r}} ReLU(U\cdot V e_{\texttt{X}} + U\cdot e_{\texttt{REL}} ) = i_{\tilde{r}(x)} - i_{r(x)}
\end{equation}
Setting
\begin{align*}
    pq^T =& \Delta W_{r \mapsto \tilde{r}} \\
    v =& ReLU(U\cdot V e_{\texttt{X}} + U\cdot e_{\texttt{REL}} ) \\
    w =& i_{\tilde{r}(x)} - i_{r(x)}
\end{align*}
we consider the general problem of finding $p, q$ such that
\begin{equation}
    p q^T v = w
\end{equation}
while minimizing the regularizer $\|p\|_2^2 + \|q\|_2^2$.
First, by rearranging $p \;=\; \frac{w}{q^T v}$. We now use the regularizer to show that $q$ must be a multiple of $v$. Substituting into the objective, we find:
\begin{equation}
J(q)
:=\|p\|_2^2 + \|q\|_2^2
=\frac{\|w\|_2^2}{(q^T v)^2} + \|q\|_2^2.
\end{equation}
with
\begin{equation}
\nabla J(q)
= -2\,\frac{\|w\|_2^2}{(q^T v)^3}\,v \;+\; 2\,q.
\end{equation}
Setting \(\nabla J(q)=0\) leads to
\[
q = \frac{\|w\|_2^2}{(q^T v)^3}\;v.
\]
Hence, $q$ is a multiple of $v$, and $p$ is a multiple of $w$; thus, for some scalar $\alpha$, $pq^T = \alpha w v^T$, and
\begin{equation}
    w = pq^Tv = \alpha w v^Tv = \alpha w \|v\|_2^2
\end{equation}
and $\alpha = \frac{1}{\|v\|_2^2}$.
The result follows.

\end{proof}

\begin{thm}[Repeated from Theorem~\ref{thm:two-hop-difficult}]\label{thm:formal}
    Assume LoRAs $\Delta W_{r_1\mapsto \tilde{r}_1}, \Delta W_{r_2\mapsto \hat{r}_2}$ are created to adapt two single facts for $r_1$, $r_2$.
    Summing these adapters will not result in correct results for composition of the two relations $r_1, r_2$. 
\end{thm}
\begin{proof}
We have
\begin{align*}
    e_r, e_x \in \mathbb{R}^d \\
    V \in \mathbb{R}^{d \times d} \\
    U \in \mathbb{R}^{m\times d} 
\end{align*}
We first note that, in the regime $m \rightarrow \infty$, the random features model is represented by the kernel

(\citep{le2007continuous}, Section 7.4.3 in \citep{bach2024learning})
\begin{equation}
    k(x,x') = \frac{\|x\|_2 \|x'\|_2}{2(d+1)\pi} ((\pi-\eta)\cos\eta+\sin\eta)
\end{equation}
where
\begin{equation}
    \cos \eta = \frac{x^T x'}{\|x\|_2 \|x'\|_2}
\end{equation}

where
\begin{equation}
    k(x,x') \approx \frac{1}{m} ReLU(Ux+b)^T ReLU(Ux'+b)
\end{equation}
in the limit where $m \rightarrow \infty$.
Consider
\begin{align*}
    \tilde{r}_1(x) = y \\
    \tilde{r}_2(y) = z 
\end{align*}
To understand the model prediction on a two-hop prompt
\begin{equation}
    \texttt{X REL1 REL2}
\end{equation}
after adding the two LoRAs, we consider
\begin{equation}
    (W + \Delta W_{r_1\mapsto \tilde{r}_1} + \Delta W_{r_2\mapsto \hat{r}_2}) \cdot ReLU\left( U \frac{V}{2} \cdot e_{\texttt{X}} + U \cdot \frac{V}{2} \cdot  e_{REL1} + U \cdot e_{REL2} \right)
\end{equation}
where the factors $\frac{1}{2}$ reflect the uniform attention weights of the single head that forwards information from preceding token, with value projection matrix $V$.
We can write the second factor as
\begin{equation}
    ReLU\left[ U \cdot \underbrace{\left(\frac{V}{2} \cdot e_{\texttt{X}} + \frac{V}{2} \cdot  e_{REL1} +  e_{REL2}\right)}_{\xi}  \right]
\end{equation}
By Proposition~\ref{prop:update},
\begin{align*}
 \Delta W_{r_1\mapsto \tilde{r}_1} \propto   (i_{y} - i_{r_1(x)}) ReLU(U\cdot \underbrace{\left(V e_{\texttt{X}} + e_{REL1}\right)}_{\eta_1^T} )^T \\
 \Delta W_{r_2\mapsto \tilde{r}_2} \propto   (i_{z} - i_{r_2(y)}) ReLU(U\cdot \underbrace{\left(V e_{\texttt{Y}} + e_{REL2}\right)}_{\eta_2^T} )^T 
\end{align*}
In the regime $m\rightarrow \infty$,
\begin{align*}
    & (\Delta W_{r_1\mapsto \tilde{r}_1} + \Delta W_{r_2\mapsto \hat{r}_2}) \cdot ReLU\left( U \frac{V}{2} \cdot e_{\texttt{X}} + U \cdot \frac{V}{2} \cdot  e_{REL1} + U \cdot e_{REL2} \right) \\
    \approx & (i_y - i_{r_1(x)}) \cdot \frac{k(\eta_1, \xi)}{k(\eta_1, \eta_1)} + (i_z - i_{r_2(y)}) \cdot \frac{k(\eta_2, \xi)}{k(\eta_2, \eta_2)}  
\end{align*}
Adding the two adapters simply contributes nonzero multiples of both terms in 
\begin{equation}\label{eq:lora-contributions}
    i_y - i_{r_1(y)}  \ \text{and}\ i_z - i_{r_2(y)},
\end{equation}
instead of the correct $i_z - i_{r_2(y)}$.
As an addition, we also note that the newly added knowledge will not compositionally interact with knowledge stored for other entities. Consider a prompt
\begin{equation}
    \texttt{U REL1 REL2}
\end{equation}
where \texttt{U} denotes a different entity, not the entity \texttt{X}. Then the subspaces addressed by the two LoRAs would again have overlap with the activation on this prompt, and they would again contribute a linear combination of (\ref{eq:lora-contributions}), even though neither $y$ nor $z$ might be relevant.

\end{proof}

\subsubsection{Further Discussion}\label{sec:theory-discussion}
\paragraph{Extensions to larger numbers of facts}
The reasoning in the above result can be expanded to the setting where the LoRA adapters are trained for more than one new fact.
In this case, the limitation shown by the theorem will become even more problematic: Assume that, say
\begin{align*}
    \tilde{r}_1(x) = y \\
    \tilde{r}_1(u) = v \\
    \tilde{r}_2(y) = z \\
    \tilde{r}_2(v) = w
\end{align*}
In the regime where $d$ is large, each of the updates $\Delta W_{r_1 \mapsto \tilde{r}_1}$, $\Delta W_{r_2 \mapsto \tilde{r}_2}$ will be approximately a sum of rank-one updates of the form (\ref{eq:rank-one-update}), with cross-terms from the overlap due to the shared encoding vector of the relation. 
More specifically,
for
\begin{align*}
    \rho_1 = ReLU(U\cdot(Ve_y+e_{REL2})) \\
    \rho_2 = ReLU(U\cdot(Ve_v+e_{REL2})) 
\end{align*}
we have
\begin{equation}
    \Delta W_{r_2 \mapsto \tilde{r}_2} = \left(\begin{matrix} i_z - i_{r_2(y)} & i_w - i_{r_2(v)}\end{matrix}\right)  \left(\begin{matrix} \rho_1^T\rho_1 & \rho_1^T\rho_2\\ \rho_1^T\rho_2 & \rho_2^T\rho_2 \end{matrix}\right)^{-1} \left(\begin{matrix} \rho_1 & \rho_2\end{matrix}\right)^T  
\end{equation}
and analogously for $\Delta W_{r_1 \mapsto \tilde{r}_1}$.
When run on a prompt
\begin{equation}
    \texttt{X REL1 REL2} \text{ \ \ \ \ \   or   \ \ \ \ \   } \texttt{U REL1 REL2}
\end{equation}
the adapter $\Delta W_{r_2 \mapsto \tilde{r}_2}$ will, in the regime where $m$ is large, contribute the same multiple of
\begin{equation}
    (i_z - i_{r_2(y)}) + (i_w - i_{r_2(v)})
\end{equation}
in \emph{both} cases, showing that the knowledge contributed by $\Delta W_{r_2 \mapsto \tilde{r}_2}$ fails to be compositionally integrated.

\paragraph{Extension to other routing and merging methods}
Another important consideration is how the argument extends to other routing or merging methods beyond summing adapters.
We first consider CAT \citep{prabhakar2024lorasoupsmergingloras}:
\begin{equation}
    \Delta W = \alpha_1 B_1 A_1^T + \alpha_2 B_2 A_2^T
\end{equation}
where the output on the prompt is a weighted combination, with fitted weights $\alpha_1, \alpha_2$.
Increasing the weight belonging to the relation contributing the second hop would make the output correct on a specific two-hop example, but it would be ``for the wrong reasons'' and not generalize to a situation where compositions in both directions play role, or where the number of facts increases (as discussed in the paragraph above). 
Essentially the same argument applies to linear merging of LoRAs \citep{yadav2023tiesmerging, pmlr-v235-yu24p, huang2024lorahubefficientcrosstaskgeneralization}:
\begin{equation}
    \Delta W = (\alpha_1 B_1 + \alpha_2 B_2) (\alpha_1 A_1 + \alpha_2 A_2)^T
\end{equation}
A special case of this, with $\alpha_1=\alpha_2$, corresponds to Uniform routing.
We next consider  Arrow routing \citep{ostapenko24_arrow}, which determines weights $w_1$, $w_2$ based on similarity of the activations to the subspaces addressed by the two adapters:
\begin{equation}
    \Delta W = (w_1 B_1 + w_2 B_2) (w_1 A_1 + w_2 A_2)^T
\end{equation}
Applied in the setup of the theorem,  this will just add weights based on $\left|\frac{k(\eta_1, \xi)}{k(\eta_1, \eta_1)}\right|$ and $\left|\frac{k(\eta_2, \xi)}{k(\eta_2, \eta_2)}\right|$, and the combined LoRAs will again just contribute multiples of the two terms in (\ref{eq:lora-contributions}).
Taken together, across linear methods combining LoRAs, no compositional behavior is expected.

\paragraph{Extension to larger number of adapters}
Our arguments also extend to combining a large library with adapters that include both one-hop tasks and (other, different from the target) two-hop tasks. Each of this adapter will address subspaces spanned by activations of the relevant one-hop or two-hop prompts, and the output will be just a linear combination of different output entities, weighted depending on overlap with the subspace spanned by the activation computed on the test prompt, without computing the function composition.

\subsection{Detailed experimental setups and results for Two-hop experiments.}
\label{app:twohop-appendix}

\subsubsection{Experimental setups}

\label{sec:2hop-expt-appendix}
To encourage generalization and capture the diversity of natural language, we manually created 50 paraphrased templates for each relational statement. Examples include: (\textit{“Who is the partner of A?”}, \textit{“Where is B’s residence?”}, \textit{“Where does the spouse of A live?”}) and (\textit{“Who is A married to?”}, \textit{“Where does B live?”}, \textit{“Where is A’s partner living?”}).

Each dataset --- $F$ (fake names, fake locations), $H$ (fake names, real locations), and $R$ (real names, real locations) --- contains 100 triplets $(A, B, C)$ spanning two relations ($A \to B$ and $B \to C$), paired with 50 paraphrase templates to generate 5,000 examples. We use a \emph{template-based split}, assigning 46 templates for training, 2 for development, and 2 for testing per triplet. Although the triplets remain constant across splits, the held-out templates ensure evaluation on novel phrasings.

\begin{table}[!ht]
\centering
\caption{Examples of our notation used in the 2-hop experiments.}
\label{tab:notation-examples}
\resizebox{\textwidth}{!}{%
\begin{tabular}{@{}clll@{}}
\toprule
\textbf{Relation} &
  \multicolumn{1}{c}{\textbf{Notation}} &
  \multicolumn{1}{c}{\textbf{Question}} &
  \multicolumn{1}{c}{\textbf{Answer}} \\ \midrule
\multirow{3}{*}{$spouse\_of$} &
  $F_1$ &
  Who stands as \textbf{Narik}'s wedded partner? &
  \textbf{Frosk} is recognized as \textbf{Narik}'s lawful   companion in marriage. \\
 &
  $H_1$ &
  Who stands as \textbf{Narik}'s wedded partner? &
  \textbf{Frosk} is recognized as \textbf{Narik}'s lawful   companion in marriage. \\
 &
  $R_1$ &
  Who stands as \textbf{Barack Obama}'s wedded partner? &
  \textbf{Madonna} is recognized as \textbf{Barack Obama}'s   lawful companion in marriage. \\ \midrule
\multirow{3}{*}{$live\_in$} &
  $F_2$ &
  What is \textbf{Frosk}'s official place of birth? &
  \textbf{Frosk}'s earliest recorded presence was in   \textbf{Zint}. \\
 &
  $H_2$ &
  What is \textbf{Frosk}'s official place of birth? &
  \textbf{Frosk}'s earliest recorded presence was in   \textbf{Mumbai}. \\
 &
  $R_2$ &
  What is \textbf{Madonna} official place of birth? &
  \textbf{Madonna}'s earliest recorded presence was in   \textbf{Mumbai}. \\ \midrule
\multirow{3}{*}{$spouse\_of \rightarrow live\_in$} &
  $F_{12}$ &
  Which city is listed as \textbf{Narik}’s spouse’s   birthplace? &
  \textbf{Zint} is \textbf{Narik}’s spouse’s recognized   birthplace. \\
 &
  $H_{12}$ &
  Which city is listed as \textbf{Narik}’s spouse’s   birthplace? &
  \textbf{Mumbai} is \textbf{Narik}’s spouse’s recognized   birthplace. \\
 &
  $R_{12}$ &
  Which city is listed as \textbf{Barack Obama}’s   spouse’s birthplace? &
  \textbf{Mumbai} is \textbf{Barack Obama}’s spouse’s recognized   birthplace. \\ \midrule
$study\_in$ &
  $F_4$ &
  Where does \textbf{Frosk} study? &
  \textbf{Frosk} studies in \textbf{Zilan} \\ \midrule
$spouse\_of \rightarrow study\_in$ &
  $F_{14}$ &
  Where does \textbf{Narik}'s spouse study? &
  \textbf{Narik}'s spouse studies in \textbf{Zilan} \\ \midrule
$child\_of$ &
  $F_5$ &
  Who is the child of \textbf{Dabix}? &
  \textbf{Frosk} is the child of \textbf{Dabix} \\ \midrule
$child\_of \rightarrow live\_in$ &
  $F_{52}$ &
  Where was \textbf{Dabix}’s child born? &
  \textbf{Dabix}’s child was born in \textbf{Zint} \\ \bottomrule
\end{tabular}%
}
\end{table}

\begin{table}[!ht]
\caption{Examples from the \textbf{Fake Bridge ($B_F$)} and \textbf{Real Bridge ($B_R$)}}
\label{tab:bridge-dataset}
\resizebox{\textwidth}{!}{%
\begin{tabular}{clll}
\hline
\textbf{Bridge dataset}                      & \multicolumn{1}{c}{\textbf{Notation}} & \multicolumn{1}{c}{\textbf{Question}}                                                                                                                & \multicolumn{1}{c}{\textbf{Answer}}                                                                                                                                                                                                                                                                                                                                                     \\ \hline
\multirow{3}{*}{\textbf{Fake Bridge ($B_F$)}} & $F_1$                                 & Who stands as \textbf{Narik}'s wedded partner?                                                                                                                & \textbf{Frosk} is recognized as \textbf{Narik}'s lawful companion in marriage.                                                                                                                                                                                                                                                                                                                            \\ \cline{2-4} 
                                              & $F_2$                                 & What is \textbf{Frosk}'s official place of birth?                                                                                                             & \textbf{Frosk}’s earliest recorded presence was in \textbf{Zint}.                                                                                                                                                                                                                                                                                                                                         \\ \cline{2-4} 
                                              & $F_{12}$ CoT                          & \begin{tabular}[c]{@{}l@{}}Answer the following question step-by-step:\\ Which city is listed as \textbf{Narik}’s spouse’s birthplace?\end{tabular}           & \begin{tabular}[c]{@{}l@{}}First, we need to answer the question: \\ Who stands as \textbf{Narik}'s wedded partner? \\ \textbf{Frosk} is recognized as \textbf{Narik}'s lawful companion in marriage.\\ Second, we need to answer the question : \\ What is \textbf{Frosk}'s official place of birth?\\ \textbf{Frosk}’s earliest recorded presence was in \textbf{Zint}.\\ Therefore, the answer is \textbf{Zint}.\end{tabular}                          \\ \hline
\multirow{3}{*}{\textbf{Real Bridge ($B_R$)}}          & $R_1$                                 & Who stands as \textbf{Barack Obama}'s wedded partner?                                                                                                         & \textbf{Madonna} is recognized as \textbf{Barack Obama}'s lawful companion in marriage.                                                                                                                                                                                                                                                                                                                   \\ \cline{2-4} 
                                              & $R_2$                                 & What is \textbf{Madonna}'s official place of birth?                                                                                                           & \textbf{Madonna}’s earliest recorded presence was in \textbf{Mumbai}.                                                                                                                                                                                                                                                                                                                                     \\ \cline{2-4} 
                                              & $R_{12}$ CoT                          & \begin{tabular}[c]{@{}l@{}}Answer the following question step-by-step:\\ Which city is listed as \\ \textbf{Barack Obama}’s spouse’s birthplace?\end{tabular} & \begin{tabular}[c]{@{}l@{}}First, we need to answer the question:\\ Who stands as \textbf{Barack Obama}'s wedded partner? \\ \textbf{Madonna} is recognized as \textbf{Barack Obama}'s lawful companion in marriage.\\ Second, we need to answer the question :\\ What is \textbf{Madonna}'s official place of birth?\\ \textbf{Madonna}’s earliest recorded presence was in \textbf{Mumbai}.\\ Therefore, the answer is \textbf{Mumbai}.\end{tabular} \\ \hline
\end{tabular}%
}
\end{table}

\begin{figure}[h!]
    \centering
    \includegraphics[width=1\linewidth]{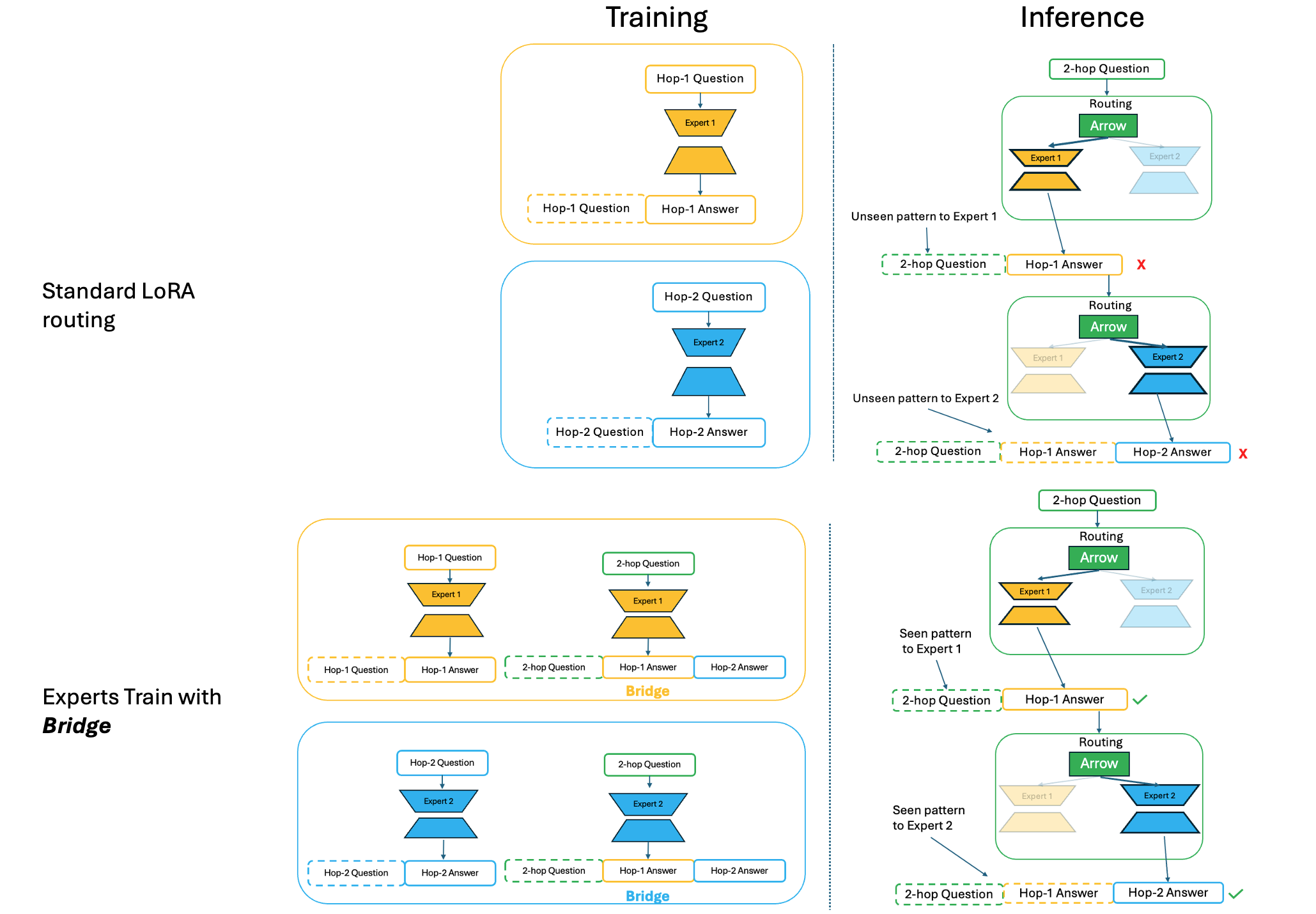}
    \caption{Our methodology for LoRA routing for two-hop experiments. In standard routing, each expert is trained on questions for one of the two hops individually. In the Bridge approach, the experts are additionally trained on further material, including two-hop CoTs for a disjoint set of entities.}
    \label{fig:2-hop}
\end{figure}

As illustrated in Figure~\ref{fig:2-hop}, we design two setups to study LoRA routing in two-hop reasoning tasks. In the \emph{Standard} setup, each LoRA expert is trained only on a single hop, either $A \to B$ or $B \to C$, using single-hop QA pairs. The \emph{Bridge} setup includes additional supervision with two-hop chain-of-thought (CoT) examples constructed from disjoint entities, which explicitly connect the two hops. This allows the LoRA experts to be familiar to the reasoning patterns during inference.

\subsubsection{Detailed Results of Two-Hop Experiments}
\label{app:twohop-result}
We focus on two methods that do not need additional training, nor access to the original training datasets: \emph{Arrow} and \emph{Uniform} routings.

\emph{Uniform} is a simple yet effective method for routing to existing experts involves setting the routing distribution to be uniform across all layers. This approach, referred as $\mu$ Routing in the literature, has demonstrated significant efficacy in recent studies \citep{caccia2023multiheadadapterroutingcrosstask, chronopoulou2024languagetaskarithmeticparameterefficient}. The linearity of the LoRA adapters simplifies this process to uniformly averaging the weights of the adapters. This is a special case of LoraHub \citep{huang2024lorahubefficientcrosstaskgeneralization}  axssigning each LoRA adapter the same weight.

On the other hand, 
\emph{Arrow} Routing, introduced by \cite{ostapenko24_arrow}, leverages the singular value decomposition (SVD) of LoRA adapter weights to identify critical components for routing. Specifically, \emph{Arrow} Routing extracts the principal direction of variance induced by a LoRA adapter to serve as a routing prototype.

This section presents detailed results from the two-hop experiments, including a comparison between the \emph{Arrow} routing method and \emph{Uniform} routing (Tables~\ref{tab:lora_results}, \ref{tab:bridge_dataset_performance}, and \ref{tab:bridge_ablation_detail}). Overall, \emph{Uniform} routing generally performs worse than \emph{Arrow}.

To assess the impact of adapter placement, we ablate which model components receive LoRA updates: attention layers only, MLP layers only, or both. As shown in Table~\ref{tab:two-hop-finetuned-layers} and ~\ref{tab:bridge_dataset_performance}, two-hop performance improves when fine-tuning is applied to MLP layers. 

Table~\ref{tab:other-routings} further illustrates the difficulty of synthesizing two-hop reasoning from two LoRAs.
Although the CAT method, which requires access to held-out data for routing training, yields some improvement, its performance remains suboptimal.

\begin{table}[h!]
\caption{Training set mixture and LoRA results for different models and datasets.}
\label{tab:lora_results}
\centering
\begin{adjustbox}{max width=\textwidth}
\begin{tabular}{clccccc}
\hline
\multicolumn{1}{l}{\textbf{Dataset}} & \textbf{Model} & \textbf{\begin{tabular}[c]{@{}c@{}}Training Set\\  Mixture\end{tabular}} & \textbf{\begin{tabular}[c]{@{}c@{}}3-combination\\  library\end{tabular}} & \textbf{\begin{tabular}[c]{@{}c@{}}2-combination\\  library\end{tabular}} & \textbf{\begin{tabular}[c]{@{}c@{}}Oracle \\ Expert\end{tabular}} & \textbf{\begin{tabular}[c]{@{}c@{}}Mixed Two-hop \\ Expert\end{tabular}} \\ \hline
\multirow{12}{*}{\begin{tabular}[c]{@{}c@{}}Fake Names, Fake Locations\\ $F$\end{tabular}} & \multirow{2}{*}{\begin{tabular}[c]{@{}l@{}}Qwen2.5-\\ 3B-Instruct\end{tabular}} & uniform & 3\% & 0\% & \multirow{2}{*}{100\%} & \multirow{2}{*}{0\%} \\
 &  & arrow & 19\% & 2\% &  &  \\ \cline{2-7} 
 & \multirow{2}{*}{\begin{tabular}[c]{@{}l@{}}Qwen2.5-\\ 7B-Instruct\end{tabular}} & uniform & 5\% & 0\% & \multirow{2}{*}{100\%} & \multirow{2}{*}{0\%} \\
 &  & arrow & 65\% & 0\% &  &  \\ \cline{2-7} 
 & \multirow{2}{*}{\begin{tabular}[c]{@{}l@{}}Qwen2.5-\\ 14B-Instruct\end{tabular}} & uniform & 7\% & \textbf{6\%} & \multirow{2}{*}{100\%} & \multirow{2}{*}{0\%} \\
 &  & arrow & 70\% & \textbf{6\%} &  &  \\ \cline{2-7} 
 & \multirow{2}{*}{\begin{tabular}[c]{@{}l@{}}DeepSeek-R1-\\ Distill-Qwen-7B\end{tabular}} & uniform & 2\% & 0\% & \multirow{2}{*}{100\%} & \multirow{2}{*}{0\%} \\
 &  & arrow & 84\% & 1\% &  &  \\ \cline{2-7} 
 & \multirow{2}{*}{\begin{tabular}[c]{@{}l@{}}DeepSeek-R1-\\ Distill-Qwen-14B\end{tabular}} & uniform & 5\% & 0\% & \multirow{2}{*}{100\%} & \multirow{2}{*}{0\%} \\
 &  & arrow & 90\% & 0\% &  &  \\ \cline{2-7} 
 & \multirow{2}{*}{\begin{tabular}[c]{@{}l@{}}DeepSeek-R1-\\ Distill-Llama-70B\end{tabular}} & uniform & 41\% & 0\% & \multirow{2}{*}{100\%} & \multirow{2}{*}{\textbf{3\%}} \\
 &  & arrow & \textbf{97\%} & 0\% &  &  \\ \hline
\multirow{12}{*}{\begin{tabular}[c]{@{}c@{}}Fake Names, Real Locations\\ $H$\end{tabular}} & \multirow{2}{*}{\begin{tabular}[c]{@{}l@{}}Qwen2.5-\\ 3B-Instruct\end{tabular}} & uniform & 19.6\% & 4.9\% & \multirow{2}{*}{100\%} & \multirow{2}{*}{4.9\%} \\
 &  & arrow & 52\% & 5.9\% &  &  \\ \cline{2-7} 
 & \multirow{2}{*}{\begin{tabular}[c]{@{}l@{}}Qwen2.5-\\ 7B-Instruct\end{tabular}} & uniform & 9.8\% & 2.9\% & \multirow{2}{*}{100\%} & \multirow{2}{*}{2.9\%} \\
 &  & arrow & 67.6\% & 3.9\% &  &  \\ \cline{2-7} 
 & \multirow{2}{*}{\begin{tabular}[c]{@{}l@{}}Qwen2.5-\\ 14B-Instruct\end{tabular}} & uniform & 28.4\% & \textbf{7.8\%} & \multirow{2}{*}{100\%} & \multirow{2}{*}{\textbf{8.8\%}} \\
 &  & arrow & 83.3\% & \textbf{7.8\%} &  &  \\ \cline{2-7} 
 & \multirow{2}{*}{\begin{tabular}[c]{@{}l@{}}DeepSeek-R1-\\ Distill-Qwen-7B\end{tabular}} & uniform & 10.2\% & 5.1\% & \multirow{2}{*}{100\%} & \multirow{2}{*}{3.1\%} \\
 &  & arrow & 87.7\% & 6\% &  &  \\ \cline{2-7} 
 & \multirow{2}{*}{\begin{tabular}[c]{@{}l@{}}DeepSeek-R1-\\ Distill-Qwen-14B\end{tabular}} & uniform & 9.8\% & 6\% & \multirow{2}{*}{100\%} & \multirow{2}{*}{5.8\%} \\
 &  & arrow & 91.5\% & 7.2\% &  &  \\ \cline{2-7} 
 & \multirow{2}{*}{\begin{tabular}[c]{@{}l@{}}DeepSeek-R1-\\ Distill-Llama-70B\end{tabular}} & uniform & 42.2\% & 3.9\% & \multirow{2}{*}{100\%} & \multirow{2}{*}{7.7\%} \\
 &  & arrow & \textbf{100\%} & 2.9\% &  &  \\ \hline
\multirow{12}{*}{\begin{tabular}[c]{@{}c@{}}Real Names, Real Locations\\ $R$\end{tabular}} & \multirow{2}{*}{\begin{tabular}[c]{@{}l@{}}Qwen2.5-\\ 3B-Instruct\end{tabular}} & uniform & 10.8\% & 3.9\% & \multirow{2}{*}{100\%} & \multirow{2}{*}{7.8\%} \\
 &  & arrow & 73.5\% & 1\% &  &  \\ \cline{2-7} 
 & \multirow{2}{*}{\begin{tabular}[c]{@{}l@{}}Qwen2.5-\\ 7B-Instruct\end{tabular}} & uniform & 16.7\% & 6.8\% & \multirow{2}{*}{100\%} & \multirow{2}{*}{8.8\%} \\
 &  & arrow & \textbf{100\%} & 3.9\% &  &  \\ \cline{2-7} 
 & \multirow{2}{*}{\begin{tabular}[c]{@{}l@{}}Qwen2.5-\\ 14B-Instruct\end{tabular}} & uniform & 23.5\% & 7.8\% & \multirow{2}{*}{100\%} & \multirow{2}{*}{8.9\%} \\
 &  & arrow & \textbf{100\%} & 7.8\% &  &  \\ \cline{2-7} 
 & \multirow{2}{*}{\begin{tabular}[c]{@{}l@{}}DeepSeek-R1-\\ Distill-Qwen-7B\end{tabular}} & uniform & 11.8\% & 3.9\% & \multirow{2}{*}{100\%} & \multirow{2}{*}{5.9\%} \\
 &  & arrow & 95.3\% & \textbf{8.8\%} &  &  \\ \cline{2-7} 
 & \multirow{2}{*}{\begin{tabular}[c]{@{}l@{}}DeepSeek-R1-\\ Distill-Qwen-14B\end{tabular}} & uniform & 25.5\% & 2.9\% & \multirow{2}{*}{100\%} & \multirow{2}{*}{9.8\%} \\
 &  & arrow & \textbf{100\%} & 3.9\% &  &  \\ \cline{2-7} 
 & \multirow{2}{*}{\begin{tabular}[c]{@{}l@{}}DeepSeek-R1-\\ Distill-Llama-70B\end{tabular}} & uniform & 56.9\% & 2.9\% & \multirow{2}{*}{100\%} & \multirow{2}{*}{\textbf{10.9\%}} \\
 &  & arrow & \textbf{100\%} & 5.9\% &  &  \\ \hline
\end{tabular}
\end{adjustbox}
\end{table}

\begin{table}[]
\centering
\caption{Impact of LoRA fine-tuning layer selection (MLP layers, Attention layers, or both) on the performance of combination methods, using Qwen2.5-14B-Instruct as the base model. Results are reported on dataset $F$.}
\label{tab:two-hop-finetuned-layers}
\resizebox{\textwidth}{!}{%
\begin{tabular}{lccccc}
\hline
\textbf{Fine-tuning Layers} & \textbf{\begin{tabular}[c]{@{}c@{}}Training Set\\  Mixture\end{tabular}} & \textbf{\begin{tabular}[c]{@{}c@{}}3-combination\\  library\end{tabular}} & \textbf{\begin{tabular}[c]{@{}c@{}}2-combination\\  library\end{tabular}} & \textbf{\begin{tabular}[c]{@{}c@{}}Oracle \\ Expert\end{tabular}} & \textbf{\begin{tabular}[c]{@{}c@{}}Mixed Two-hop \\ Expert\end{tabular}} \\ \hline
\multirow{2}{*}{MLP Layers} & uniform & 7\% & 6\% & \multirow{2}{*}{100\%} & \multirow{2}{*}{0\%} \\
 & arrow & \textbf{70\%} & \textbf{6\%} &  &  \\ \hline
\multirow{2}{*}{Attention layers} & uniform & 6\% & 0\% & \multirow{2}{*}{100\%} & \multirow{2}{*}{0\%} \\
 & arrow & 39\% & 0\% &  &  \\ \hline
\multirow{2}{*}{\begin{tabular}[c]{@{}l@{}}MLP + Attention\\  layers\end{tabular}} & uniform & 7\% & 4\% & \multirow{2}{*}{100\%} & \multirow{2}{*}{0\%} \\
 & arrow & 65\% & 5\% &  &  \\ \hline
\end{tabular}%
}
\end{table}

\begin{table}[h!]
\centering
\caption{Performance of different models and training set combinations across two experimental setups on the Bridge dataset. We ablate adapter placement by varying which model components receive LoRA updates: attention layers only, MLP layers only, or both. Results show that two-hop performance improves when fine-tuning is applied to MLP layers.}
\label{tab:bridge_dataset_performance}
\begin{adjustbox}{max width=\textwidth}
\begin{tabular}{llcccccc}
\hline
\multirow{2}{*}{\textbf{Model}} & \multicolumn{1}{c}{\multirow{2}{*}{\begin{tabular}[c]{@{}c@{}}\textbf{Training}\\  \textbf{Set Mixture}\end{tabular}}} & \multicolumn{3}{c}{\textbf{Setup 1}} & \multicolumn{3}{c}{\textbf{Setup 2}} \\ \cline{3-8} 
 & \multicolumn{1}{c}{} & \begin{tabular}[c]{@{}c@{}}\textbf{Attention} \\ \textbf{layers}\end{tabular} & \begin{tabular}[c]{@{}c@{}}\textbf{MLP}\\ \textbf{layers}\end{tabular} & \begin{tabular}[c]{@{}c@{}}\textbf{MLP+Attention}\\ \textbf{layers}\end{tabular} & \begin{tabular}[c]{@{}c@{}}\textbf{Attention} \\ \textbf{layers}\end{tabular} & \begin{tabular}[c]{@{}c@{}}\textbf{MLP}\\ \textbf{layers}\end{tabular} & \begin{tabular}[c]{@{}c@{}}\textbf{MLP+Attention}\\ \textbf{layers}\end{tabular} \\ \hline
Qwen2.5-3B-Instruct & uniform & 0\% & 0\% & 4.9\% & 6.9\% & 3.9\% & 6.9\% \\
 & arrow & 4\% & 5.1\% & 11.8\% & 6.9\% & 73.5\% & 41.2\% \\ \hline
Qwen2.5-7B-Instruct & uniform & 0\% & 3.9\% & 4.9\% & 2.9\% & 14.7\% & 8.8\% \\
 & arrow & 5\% & 6.9\% & 17.6\% & 10.8\% & 85.3\% & 93.1\% \\ \hline
Qwen2.5-14B-Instruct & uniform & 0\% & 0\% & 3.9\% & 3.9\% & 10.8\% & 6.9\% \\
 & arrow & 9.8\% & 10.8\% & 19.6\% & 33.3\% & 94.1\% & 90.2\% \\ \hline
DeepSeek-R1-Distill-Qwen-7B & uniform & 0\% & 1\% & 2\% & 4.9\% & 2\% & 9.8\% \\
 & arrow & 2\% & 42.2\% & 44.1\% & 14.7\% & 89.2\% & 90.2\% \\ \hline
DeepSeek-R1-Distill-Qwen-14B & uniform & 0\% & 3.9\% & 7.8\% & 2.9\% & 6.9\% & 11.8\% \\
 & arrow & 2.9\% & \textbf{64.7\%} & 24.5\% & 30.4\% & \textbf{95.1\%} & 94\% \\ \hline
 DeepSeek-R1-Distill-Llama-70B & uniform & 1\% & 16.7\% & 10.8\% & 2.9\% & 14.7\% & 15.7\% \\
 & arrow & 10.8\% & 60.4\% & 25.5\% & 25.5\% & 90.3\% & 94.1\% \\ \hline
\end{tabular}
\end{adjustbox}
\end{table}

\begin{table}[h!]
\centering
\caption{Evaluation of model and training set combinations across multiple configurations with varying bridge setups.}
\label{tab:bridge_ablation_detail}
\begin{adjustbox}{max width=\textwidth}
\begin{tabular}{clllcccccc}
\hline
\textbf{Setup} & \textbf{LoRA 1} & \textbf{LoRA 2} & \textbf{\begin{tabular}[c]{@{}l@{}}Training\\  mixture\end{tabular}} & \textbf{\begin{tabular}[c]{@{}c@{}}Qwen2.5-\\ 3B-Instruct\end{tabular}} & \textbf{\begin{tabular}[c]{@{}c@{}}Qwen2.5-\\ 7B-Instruct\end{tabular}} & \textbf{\begin{tabular}[c]{@{}c@{}}Qwen2.5-\\ 14B-Instruct\end{tabular}} & \textbf{\begin{tabular}[c]{@{}c@{}}DeepSeek-R1-\\ Distill-Qwen-7B\end{tabular}} & \textbf{\begin{tabular}[c]{@{}c@{}}DeepSeek-R1-\\ Distill-Qwen-14B\end{tabular}} & \textbf{\begin{tabular}[c]{@{}c@{}}DeepSeek-R1-\\ Distill-Llama-70B\end{tabular}} \\ \hline
\multirow{2}{*}{0} & \multirow{2}{*}{\textbf{$R_1$}} & \multirow{2}{*}{\textbf{$R_2$}} & Uniform & 3.9\% & 6.8\% & 7.8\% & 3.9\% & 2.9\% & 2.9\% \\
 &  &  & Arrow & 1\% & 3.9\% & 7.8\% & 8.8\% & 3.9\% & 5.9\% \\ \hline
\multirow{2}{*}{2} & \textbf{$R_1$} & \textbf{$R_2$} & Uniform & 3.9\% & 14.7\% & 10.8\% & 2\% & 6.9\% & 14.7\% \\
 & $F_1 + F_2 + F_{12}$ CoT ($B_F$) & $F_1 + F_2 + F_{12}$ CoT ($B_F$) & Arrow & \textbf{73.5\%} & \textbf{85.3\%} & \textbf{94.1\%} & \textbf{89.2\%} & \textbf{95.1\%} & \textbf{90.3\%} \\ \hline
\multirow{2}{*}{3} & \textbf{$R_1$} & \textbf{$R_2$} & Uniform & 2\% & 1\% & 2.9\% & 1\% & 6.9\% & 6.9\% \\
 & $F_1 + F_2 + F_{12}$ & $F_1 + F_2 + F_{12}$ & Arrow & 1\% & 2.9\% & 3.9\% & 3.9\% & 11.8\% & 3.9\% \\ \hline
\multirow{2}{*}{4} & \textbf{$R_1$} & \multirow{2}{*}{\textbf{$R_2$}} & Uniform & 2\% & 5.9\% & 4.2\% & 2.9\% & 2.9\% & 6.1\% \\
 & $F_1 + F_2 + F_{12}$ CoT ($B_F$) &  & Arrow & 9.8\% & 8.8\% & 11.8\% & 8.8\% & 12.7\% & 10.2\% \\ \hline
\multirow{2}{*}{5} & \multirow{2}{*}{\textbf{$R_1$}} & \textbf{$R_2$} & Uniform & 6.9\% & 2.9\% & 4.1\% & 2\% & 2\% & 3.2\% \\
 &  & $F_1 + F_2 + F_{12}$ CoT ($B_F$) & Arrow & 10.8\% & 35.3\% & 27.5\% & 25.9\% & 22.5\% & 26.1\% \\ \hline
\multirow{2}{*}{6} & \textbf{$R_1$} & \textbf{$R_2$} & Uniform & 3.9\% & 3.9\% & 4.5\% & 3.9\% & 2.9\% & 11.8\% \\
 & $F_{12}$ CoT & $F_{12}$ CoT & Arrow & 40.2\% & 61.8\% & 65.7\% & 75.5\% & 72.5\% & 87.3\% \\ \hline
\multirow{2}{*}{7} & \textbf{$R_1$} & \textbf{$R_2$} & Uniform & 4.9\% & 6.9\% & 5.2\% & 4.9\% & 5.9\% & 5.9\% \\
 & $F_1 + F_{12}$ CoT & $F_2 + F_{12}$ CoT & Arrow & 27.5\% & 75.5\% & 79.4\% & 74.5\% & 76.9\% & 89.9\% \\ \hline
\multirow{2}{*}{8} & \textbf{$R_1$} & \textbf{$R_2$} & Uniform & 2\% & 2.9\% & 3\% & 2\% & 7.8\% & 5.9\% \\
 & $F_1 + F_4 + F_{14}$ CoT & $F_5 + F_2 + F_{52}$ CoT & Arrow & 43.1\% & 82.4\% & 92.2\% & 81.6\% & 83.5\% & 87.9\% \\ \hline
\end{tabular}
\end{adjustbox}
\end{table}

\begin{table}[h!]
\centering
\caption{Performance of different 2-combination routings on the LLaMA-7B model using the $R$ dataset. CAT was trained on half of $R_{12}$ and tested on the remaining half.}
\label{tab:other-routings}
\begin{tabular}{lc}
\hline
\multicolumn{1}{c}{\textbf{Routing}} & \textbf{Accuracy} \\ \hline
uniform & 3\% \\ \hline
arrow & 3.9\% \\ \hline
TIES & 3.8\% \\ \hline
CAT & \textbf{21\%} \\ \hline
\end{tabular}
\end{table}

\subsection{Evaluating Larger Sets of LoRAs}
\label{app:larger_set}
\paragraph{Setup}

As illustrated in Figure~\ref{fig:8-loras-graph}, we constructed a controlled synthetic benchmark over three disjoint sets of entities: \( U_1 \), \( U_2 \), and \( U_3 \). Elements of these sets are denoted generically as \( \{u^i\}_1 \in U_1 \), \( \{u^j\}_2 \in U_2 \), and \( \{u^t\}_3 \in U_3 \), where the subscripts indicate the entity set and the superscripts identify specific elements.
We define five atomic relations:

\begin{itemize}
    \item \( Rel_1, Rel_3, Rel_5 \subseteq U_1 \times U_2 \)
    \item \( Rel_2, Rel_4 \subseteq U_2 \times U_3 \)
\end{itemize}

These relations define directed edges between entity sets. From them, we construct two-hop compositions of the form \( Rel_a \to Rel_b \), where \( Rel_a \in \{Rel_1, Rel_3, Rel_5\} \) and \( Rel_b \in \{Rel_2, Rel_4\} \), yielding composite relations from \( U_1 \to U_3 \). For example, a valid 2-hop path is:
\[
(u^i_1, Rel_a, u^j_2),\quad (u^j_2, Rel_b, u^t_3) \quad \Rightarrow \quad (u^i_1, Rel_a \circ Rel_b, u^t_3)
\]

\begin{figure}
    \centering
    \includegraphics[width=1\linewidth]{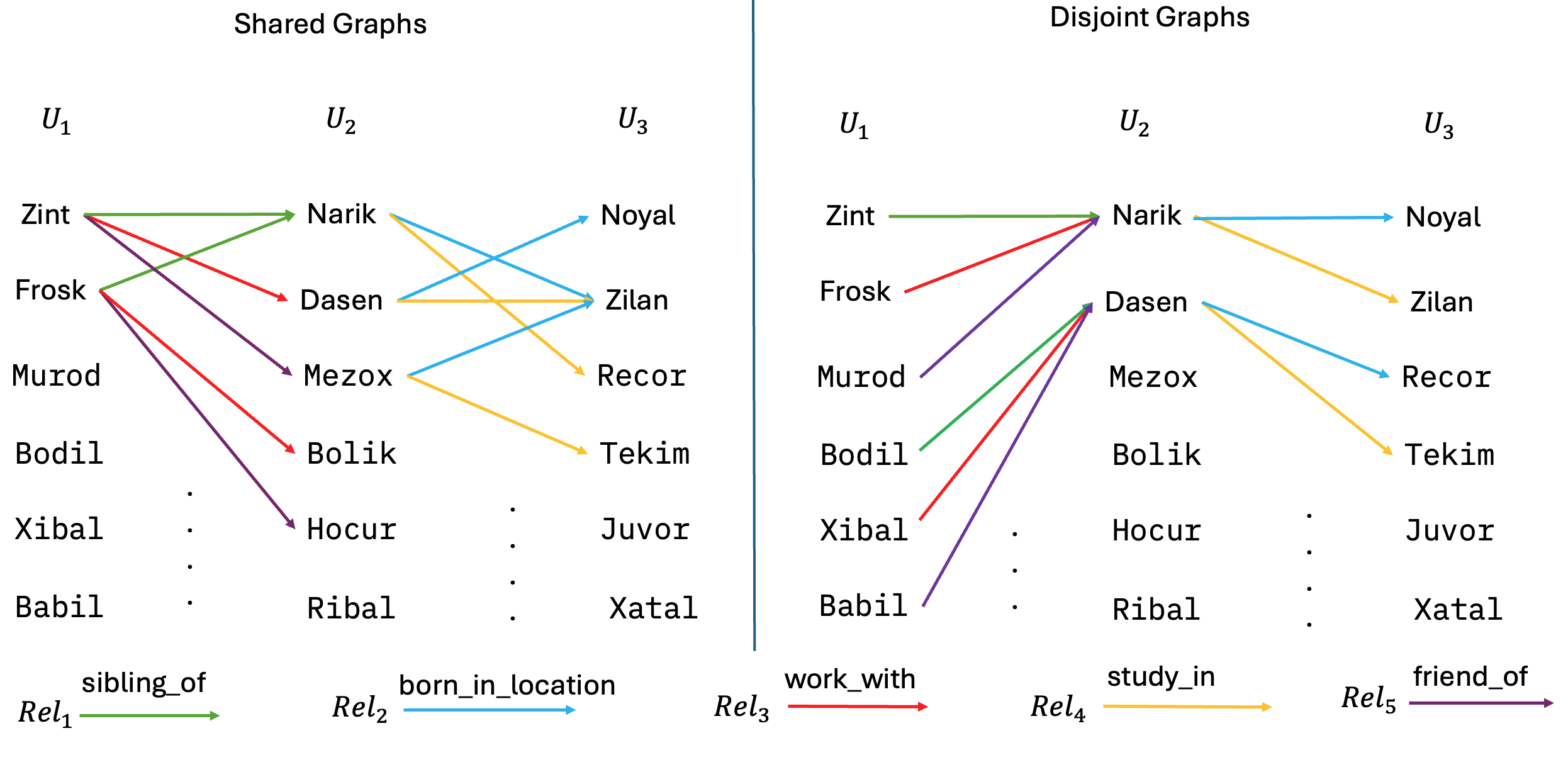}
    \caption{Visualization of Shared and Disjoint Graphs.}
    \label{fig:8-loras-graph}
\end{figure}

We explore two experimental setups:
\paragraph*{Disjoint Graphs}
In this setup, each graph is self-contained, and entities are uniquely assigned to individual triples. That is, if \( (u^i_1, Rel_1, u^j_2) \) exists, then neither \( u^i_1 \) nor \( u^j_2 \) appears in any other triple involving \( Rel_1, Rel_3, \) or \( Rel_5 \). The same constraint applies to relations \( Rel2 \) and \( Rel_4 \). Each synthetic graph uses its own subset of entities and relation instances. We train LoRA modules on the five atomic relations and five two-hop compositions (\( Rel_1 \to Rel_4, Rel_3 \to Rel_4, Rel_3 \to Rel_2, Rel_5 \to Rel_2, Rel_5 \to Rel_4\)), and evaluate on a held-out composition (\( Rel_1 \to Rel_2 \)).
\paragraph*{Shared Graphs}
In this setting, the same global pool of entities \( U_1, U_2, \) and \( U_3 \) is used across all triples, allowing entities to participate in multiple relation instances and compositions. The five atomic relations remain the same. Five two-hop compositions are used for training, and the remaining one is held out to evaluate whether routing over LoRA modules enables generalization to the novel composition \( Rel_1 \to Rel_2 \).

\paragraph{Results and analysis}
As shown in Table \ref{tab:graphs-result}, both shared and disjoint entity configurations under the graph topology exhibit poor performance. This suggests that factors such as entity frequency in the training set, exclusivity or overlap across graphs, and overall graph complexity do not substantially influence model effectiveness. Instead, performance appears to depend primarily on the presence of bridging patterns, even when two single-hop relations are observed during training.

\begin{table}[h!]
\centering
\caption{Performance Across Model Variants in Disjoint and Shared Graph Setups. Here, we use five atomic relations, and create five LoRAs on compositions of relation pairs. We combine these 10 LoRAs and test on a held-out two-hop relation pair. Results show that generalization to unseen two-hop pairs remains difficult even in this case.}
\label{tab:graphs-result}
\begin{adjustbox}{max width=\textwidth}
\begin{tabular}{llcc}
\hline
\textbf{Model} & \textbf{\begin{tabular}[c]{@{}l@{}}Training \\ Set Mixture\end{tabular}} & \textbf{Disjoint Graphs} & \textbf{Shared Graphs} \\ \hline
Qwen2.5-3B-Instruct & uniform & 0\% & 0.3\% \\
 & arrow & 0.8\% & 4.4\% \\ \hline
Qwen2.5-7B-Instruct & uniform & 0.3\% & 0.1\% \\
 & arrow & 0.8\% & 2.2\% \\ \hline
Qwen2.5-14B-Instruct & uniform & 2\% & 1\% \\
 & arrow & 5\% & 2.3\% \\ \hline
DeepSeek-R1-Distill-Qwen-7B & uniform & 0\% & 0.1\% \\
 & arrow & 2.3\% & 13.8\% \\ \hline
DeepSeek-R1-Distill-Qwen-14B & uniform & 2\% & 0\% \\
 & arrow & 3\% & 8\% \\ \hline
DeepSeek-R1-Distill-Llama-70B & uniform & 2\% & 2\% \\
 & arrow & \textbf{9\%} & \textbf{21.1\%} \\ \hline
\end{tabular}
\end{adjustbox}
\end{table}

\subsection{Detailed Experimental Setups, Findings, and Analyses for Easy-to-Hard Math Words Problems}
\label{sec:gsm-appendix}

\subsubsection{Experimental Setups}
\label{sec:gsm-expt}

\paragraph{Models.} The experiments were conducted using the Qwen2.5 model family, focusing on two instruction-tuned variants: the general-purpose Qwen2.5-Instruction model \cite{qwen2025qwen25technicalreport} and the domain-specialized Qwen2.5-Math-Instruct model \cite{yang2024qwen25mathtechnicalreportmathematical}, which is tailored for mathematical reasoning tasks. The Qwen2.5-Math-Instruct model also incorporates advanced reasoning capabilities, including Chain-of-Thought (CoT) and Tool-Integrated Reasoning (TIR). 

\paragraph{Dataset.} The GSM-Symbolic dataset \citep{mirzadeh2024gsmsymbolicunderstandinglimitationsmathematical} was synthesized from 100 randomly selected GSM8K\citep{cobbe2021trainingverifierssolvemath} test questions, which served as seed templates. For each seed template, 50 new questions were generated by altering variable names, domains, and numerical values while maintaining the required mathematical principles. Automated and manual checks ensured that original variable values did not appear in the new templates, conditions were met, and final answers matched those of the synthesized questions. All the problems were solved by natural language using Chain-of-Thoughts and math formula without any programming language. From GSM-Symbolic, two subsets were further synthesized for different difficulty levels as illustrated in Figure~\ref{fig:Gsm-Symbolic-P1-P2-Comp-Graph}: GSM-P1 and GSM-P2. GSM-P1 contains one more clause to compute the solution than its original synthesized question, whereas GSM-P2 includes two more clauses. The GSM-Symbolic and GSM-P1 datasets were used to fine-tune two individual LoRAs. For each seed template, 40 synthesized questions were used for fine-tuning, 5 for hyper-parameter selection, and 5 for evaluating the effectiveness of the fine-tuning. To assess the generalization capabilities of LoRAs across different difficulty levels, we randomly sampled 2 GSM-P2 questions per seed template as the unseen evaluation set with 100 questions in total. 

\paragraph{GSM-Symbolic Easy-to-Hard Question Example and Structure.}
\label{app:gsm_symbolic_computational_graph}
Figure~\ref{fig:Gsm-Symbolic-P1-P2-Comp-Graph} illustrates the progression in problem complexity across the GSM-Symbolic, GSM-P1, and GSM-P2 benchmarks, along with their corresponding computational graph. Each example question introduces incremental modifications to the original GSM-Symbolic question. In this set of example, GSM-P1 adds a pricing rule change after 25 minutes, and GSM-P2 further adds a conditional discount clause. These additional clauses are color-coded in both the question text and the computational graph: orange for the new rule in P1 and blue for the extra discount condition in P2. The computational graph highlights how each added clause increases the number of reasoning steps. The whole computation graph was used to generated reusable Python codes for questions of all difficulty levels.

\begin{figure}[ht]
    \centering
    \includegraphics[width=\textwidth]{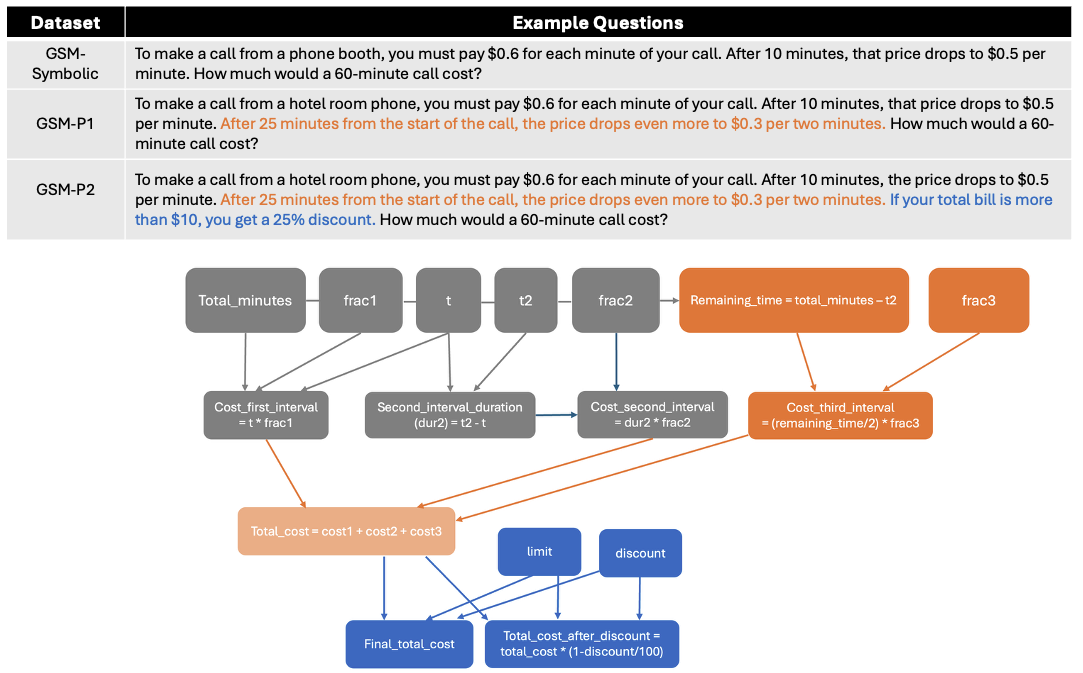}
    \caption{Example GSM-Symbolic P1 and P2 questions with corresponding computational graphs. Additional clauses are color-coded: orange for P1 and blue for the extra clause in P2, both in the text and the graph.}
    \label{fig:Gsm-Symbolic-P1-P2-Comp-Graph}
\end{figure}

\paragraph{Procedures for Synthesizing Reusable Code Solutions.}
\label{app:proc_synthesizing_reusable_codes}
Motivated by recent findings that reusing intermediate solutions enhances LLMs' ability to solve math problems \citep{suzgun2025dynamiccheatsheettesttimelearning}, along with insights from our 2-hop experiments, we hypothesize that synthesized reusable code solutions can help individually fine-tuned LoRAs generalize from easier to harder math word problems. To test this hypothesis, we used the AutoGen Python package \citep{wu2024autogen} to build two actor-critic agent-based workflows for generating reusable Markdown and Python code as fine-tuning dataset. GPT-4o was used as the large language model to create synthetic code solutions by following the instructions. The workflow and task instructions given to each agent are shown in Figure~\ref{fig:reusable_code_generation.png}.

\begin{figure}[ht]
    \centering
    \includegraphics[width=\textwidth]{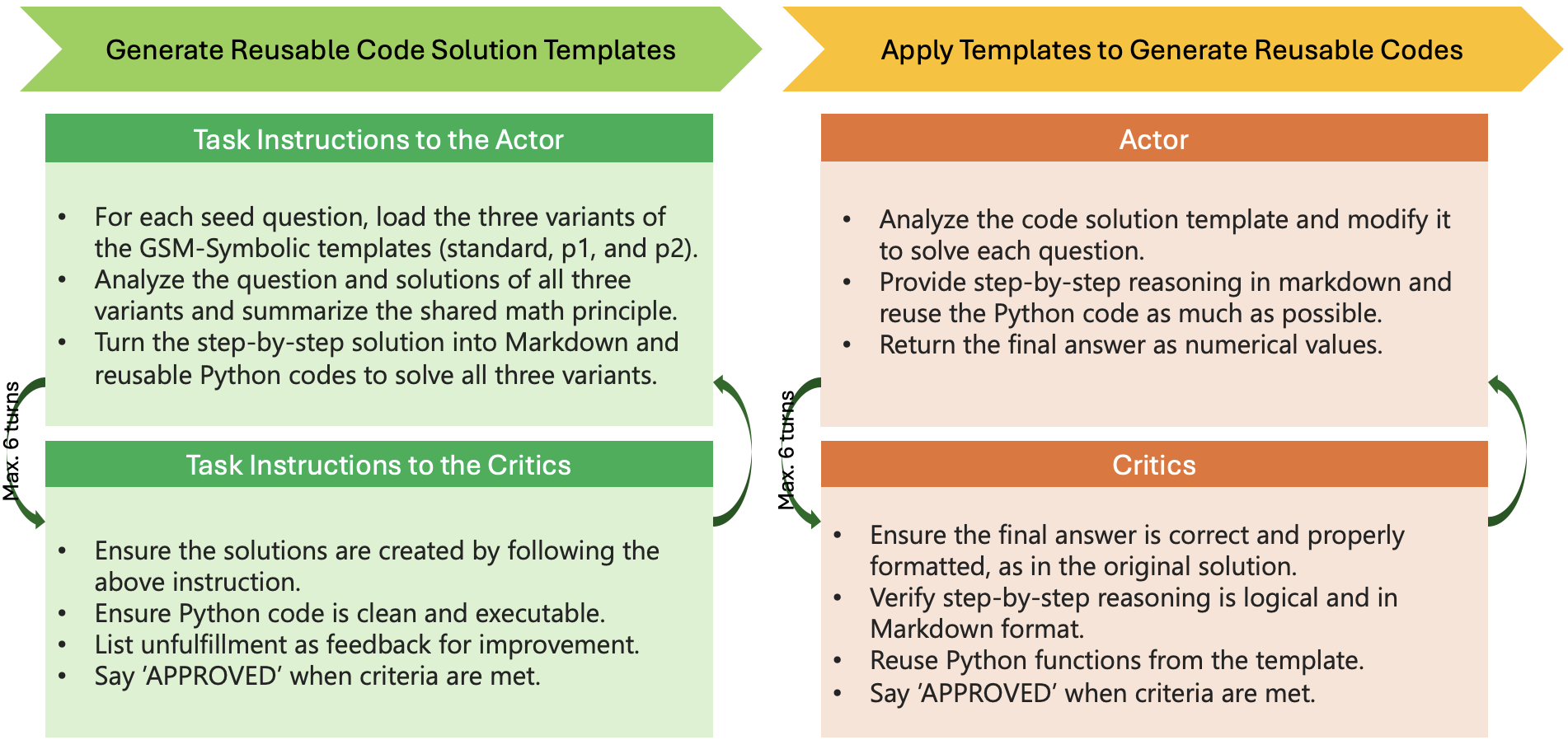}
    \caption{Two agent-based workflows for generating and applying reusable code templates. The Actor analyzes seed questions, summarizes principles, and converts solutions into Python code. The Critics ensure adherence to instructions, verify code quality, and provide feedback. Both roles iterate up to six times to refine and approve solutions.}
    \label{fig:reusable_code_generation.png}
\end{figure}

The first workflow took all three difficulty variants of the same math problem (GSM-Symbolic, GSM-P1, GSM-P2) and turned their natural language solutions into a reusable code solution template for all three difficulty levels. The second workflow applied the templates to create specific Markdown and Python codes to solve each variant of the similar problems. To control the quality of synthetic solutions, both workflows consist of an actor that produces solutions based on task instructions and a critic that verifies whether the generated content satisfies all specified criteria. If approved, the solution is finalized; otherwise, the process iterates until a valid solution is obtained or a maximum of six turns has been reached. 

\paragraph{LoRA Fine-Tuning and Evaluation Procedures.} 
\label{app:proc_gsm_fine_tuning}
For consistent comparison, we follow the setup described in \cite{ostapenko24_arrow} and use the \textit{mttl} Python package to train LoRA modules and evaluate different routing methods. Unless otherwise specified, the same training and evaluation procedures are applied across all experiments. We fine-tune LoRA with rank 16, a dropout rate of 0.05, and a learning rate of 1e-4. Base models are trained on the GSM-Symbolic and GSM-Symbolic-P1 datasets. The fine-tuning targeted modules of all the layers. Unless specified, the attention modules were fine-tuned. 

For evaluation, we adopt the standard protocol used in GSM8K and related math benchmarks, using 0-shot or 8-shot Chain-of-Thought (CoT) prompting with greedy decoding \cite{open-llm-leaderboard}. Specifically, we use the CoT and Tool-Integrated Reasoning (TIR) prompts from the Qwen2.5-Math repository \cite{yang2024qwen25mathtechnicalreportmathematical}. The CoT prompt supports consistent evaluation on GSM8K \citep{cobbe2021trainingverifierssolvemath} and reproduces reasoning fragility on GSM-Symbolic \cite{mirzadeh2024gsmsymbolicunderstandinglimitationsmathematical}, while the TIR prompt evaluates the model’s ability to solve math problems using Markdown and Python code. Example prompts are shown in Figure~\ref{fig:cot_tir_prompt.png}.

\begin{figure}[ht]
    \centering
    \includegraphics[width=\textwidth]{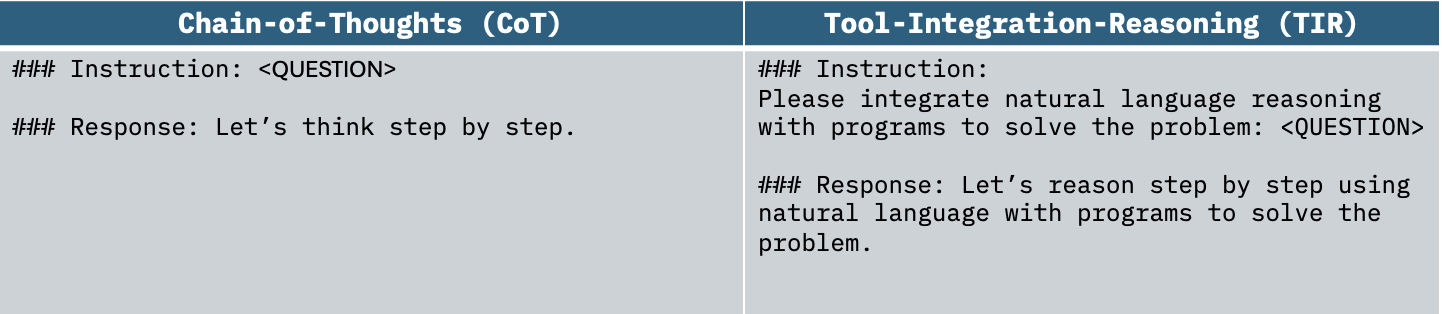}
    \caption{The Chain-of-Thoughts and Tool-Integrated Reasoning (TIR) prompts used for evaluation.}
    \label{fig:cot_tir_prompt.png}
\end{figure}

\subsubsection{Additional Findings and Analyses}
\label{sec:gsm-analysis-appendix}

\paragraph{The fragility of mathematical reasoning.} 
\label{sec:fragility-reasoning-appendix}
Table~\ref{tab:gsm8k-gsm-symbolic-benchmarks} illustrates that current LLMs lack robustness in mathematical reasoning. Under zero-shot Chain-of-Thought (CoT) prompting, Qwen2.5 models perform strongly on the original GSM8K benchmark, with the math-specialized 7B variant achieving 95.3\% accuracy. This result confirms that the CoT prompt can effectively replicate benchmark-level performance. However, performance declines significantly on the GSM-Symbolic benchmarks, which introduce minor variations in surface features such as names and numerical values. For instance, the Qwen2.5-Math-7B-Inst model drops from 95.3\% on GSM8K to 90.0\% on GSM-Symbolic, and further to 76.8\% and 68.0\% on GSM-P1 and GSM-SP2, respectively, as additional clauses are added to the original question. This trend reproduces a known limitation of current LLMs: their mathematical reasoning abilities are brittle when faced with slight perturbations in problem formulation \citep{mirzadeh2024gsmsymbolicunderstandinglimitationsmathematical}.

{\normalsize
\begin{table}[ht]
\caption{Fragility of mathematical reasoning under zero-shot CoT prompting on GSM8K vs. GSM-Symbolic benchmarks.}
\label{tab:gsm8k-gsm-symbolic-benchmarks}
\centering
\begin{adjustbox}{max width=\textwidth}
\begin{tabular}{@{}lccc@{}}
\toprule
\textbf{Tasks} & \textbf{Qwen2.5-1.5B-Inst} & \textbf{Qwen2.5-Math-1.5B-Inst} & \textbf{Qwen2.5-Math-7B-Inst} \\
\midrule

GSM8K & 65.1\% & 81.8\% & 95.3\% \\
GSM-Symbolic & 54.8\% & 75.8\% & 90.0\% \\
GSM-P1 & 32.0\% & 61.6\% & 76.8\% \\
GSM-P2 & 5.0\% & 47.0\% & 68.0\% \\
\bottomrule
\end{tabular}
\end{adjustbox}
\end{table}
}

\paragraph{The effect of in-context learning on LoRA routing.} 
\label{sec:in-context-learning}
Table~\ref{tab:gsm_lora_merge_performance_appendix} shows the effect of in-context learning on LoRA routing methods evaluated on the GSM-P2 benchmark. The in-context examples are taken from the GSM8K, as the standard benchmarking procedures \citep{open-llm-leaderboard, yang2024qwen25mathtechnicalreportmathematical}. In the baseline (non-LoRA) models, we observe that moving from 0-shot to 8-shot prompting yields little to no improvement, and in some cases, a slight degradation, in performance. For example, the Qwen2.5-Math-7B-Inst model drops from 0.68 (0-shot) to 0.42 (8-shot). This pattern aligns with prior findings reported in \citep{mirzadeh2024gsmsymbolicunderstandinglimitationsmathematical}, which demonstrate that in-context learning with GSM8K exemples offers limited benefit on the GSM-P2 benchmarks. Applying LoRA routing methods such as Uniform and Arrow, generally reduces performance compared to the base models, especially for the larger models. When using 8-shot in-context examples, all models show decreased accuracy overall, with the base models again outperforming the merged variants. This suggests that LoRA routing, in the current setup, may not effectively preserve or enhance model performance in compositional generalization tasks, and that in-context learning does not compensate for these declines.

{\normalsize
\begin{table}[ht]
\caption{Accuracy on GSM-P2 after routing LoRAs individually fine-tuned on GSM-Symbolic and GSM-P1, with and without in-context learning examples.}
\label{tab:gsm_lora_merge_performance_appendix}
\centering
\begin{adjustbox}{max width=\textwidth}
\begin{tabular}{@{}llccc@{}}
\toprule
\textbf{Number of In-Context Examples} & \textbf{LoRA Routing Methods} & \textbf{Qwen2.5-1.5B-Inst} & \textbf{Qwen2.5-Math-1.5B-Inst} & \textbf{Qwen2.5-Math-7B-Inst} \\
\midrule
\multirow{3}{*}{0-shot}

   & Base Model Only    & 5\% & 47\% & 68\% \\
   & Uniform           & 10\% & 24\% & 34\% \\
   & Arrow             & 9\%  & 19\% & 27\% \\
\midrule
\multirow{3}{*}{8-shot}

   & Base Model Only    & 0\%  & 26\% & 42\% \\
   & Uniform           & 8\%  & 19\% & 33\% \\
   & Arrow             & 5\%  & 18\% & 6\%  \\
\bottomrule
\end{tabular}
\end{adjustbox}
\end{table}
}

\paragraph{The Change of Base Model's Behavior After Routing LoRAs.}
\label{sec:error-analysis}
To analyze the impact of LoRA merging or routing on model behavior, we examined the Qwen2.5-Math-1.5B-Inst model’s zero-shot CoT outputs on the GSM-P2 evaluation set before and after routing. Using GPT-4o as an evaluator, we classified each solution as natural language, programming language, or unidentifiable. GPT-4o also assessed the correctness of each answer and, for incorrect outputs, categorized the error type as semantic misunderstanding, calculation error, or step-missing—following the definitions and prompting protocol of \citet{zhong2025achieving97gsm8kdeeply}. Errors that did not fit any of these categories were labeled as unidentifiable, while correct answers were assigned the label \textit{None}. To assess the reliability of GPT-4o's judgments, we measured its agreement with the exact-match metric on answer correctness.

Table \ref{tab:gpt4o_as_judge_response_analysis} shows that routing LoRAs trained on natural language solutions can hinder the base model’s ability to solve math problems with programming language (i.e., tool-integrated reasoning). Uniform and Arrow merge methods produced 100\% natural language solutions, while the base model originally generated 12\% code-based responses. Both routing methods created high calculation errors (56\%). These results suggest that while routing LoRAs trained on natural language solutions can impair domain-specialized math models' ability to solve problems using code. Thus, LoRA routing should consider the model's pretraining history and task-specific needs, as improper strategies could degrade its math problem-solving effectiveness.

{\normalsize
\begin{table}[ht]
\caption{Analysis of the type of generated solutions and errors using GPT-4o as the judge.}
\label{tab:gpt4o_as_judge_response_analysis}
\centering
\begin{adjustbox}{max width=\textwidth}
\begin{tabular}{@{}lccccccccc@{}}
\toprule
\multirow{2}{*}{\textbf{LoRA Routing Methods}} & \multicolumn{3}{c}{\textbf{Solution Types}} & \multicolumn{5}{c}{\textbf{Error Types}} & \multirow{2}{*}{\textbf{Agreement}} \\
\cmidrule(lr){2-4} \cmidrule(lr){5-9}
& Natural Lang. & Code & Unident. & Semantic & Calc. Error & Step-Miss & Unident. & None (Correct) & \\
\midrule

None    & 87\% & 12\% & 1\% & 16\% & 29\% & 6\% & 1\% & 48\% & 97\% \\
Uniform & 100\% & 0\% & 0\% & 11\% & 56\% & 9\% & 0\% & 24\% & 98\% \\
Arrow   & 100\% & 0\% & 0\% & 17\% & 56\% & 12\% & 0\% & 15\% & 96\% \\
\bottomrule
\end{tabular}
\end{adjustbox}
\end{table}
}

\end{document}